\documentclass[twoside]{article}

\usepackage[accepted]{aistats2025}

\usepackage{amsmath, amsthm, amssymb, bm}
\newtheorem{theorem}{Theorem}[section]
\newtheorem{assumption}{Assumption}[section]
\newtheorem{proposition}{Proposition}[section]
\newtheorem{remark}{Remark}[section]

\newtheorem{lemma}[theorem]{Lemma}

\def\E{\mathbb E}
\def\R{\mathbb R}
\def\S{\mathbb S}

\def\N{\mathcal N}

\def\bx{\mathbf x}
\def\bw{\mathbf w}
\def\bv{\mathbf v}
\def\bu{\mathbf u}
\def\ba{\bm{\alpha}}
\def\bx{\mathbf x}

\usepackage{xcolor}

\usepackage[round]{natbib}

\usepackage{graphicx}
\usepackage{subfig}
\usepackage{hyperref} 
\usepackage{enumitem}

\begin{document}


\runningtitle{Learning Gaussian Multi-Index Models with Gradient Flow}

%

\twocolumn[

\aistatstitle{{\Large Learning Gaussian Multi-Index Models with Gradient Flow:} \\ {\Large Time Complexity and Directional Convergence}}

\aistatsauthor{ Berfin \c{S}im\c{s}ek \And Amire Bendjeddou \And  Daniel Hsu }

\aistatsaddress{ Flatiron Institute \And  EPFL \And Columbia University }]

\begin{abstract}
  This work focuses on the gradient flow dynamics of a neural network model that uses correlation loss to approximate a multi-index function on high-dimensional standard Gaussian data. 
  Specifically, the multi-index function we consider is a sum of neurons $f^*(\bx) \!=\! \sum_{j=1}^k \! \sigma^*(\bv_j^T \bx)$ where $\bv_1, \dots, \bv_k$ are unit vectors, and $\sigma^*$ lacks the first and second Hermite polynomials in its Hermite expansion. 
  It is known that, for the single-index case ($k\!=\!1$), overcoming the search phase requires polynomial time complexity. 
  We first generalize this result to multi-index functions characterized by vectors in arbitrary directions.  
  After the search phase, it is not clear whether the network neurons converge to the index vectors, or get stuck at a sub-optimal solution. 
  When the index vectors are orthogonal, we give a complete characterization of the fixed points and prove that neurons converge to the nearest index vectors. 
  Therefore, using $n \! \asymp \! k \log k$ neurons ensures finding the full set of index vectors with gradient flow with high probability over random initialization. 
  When $\bv_i^T \bv_j \!=\! \beta \! \geq \! 0$ for all $i \neq j$, we prove the existence of a sharp threshold $\beta_c \!=\! c/(c+k)$ at which the fixed point that computes the average of the index vectors transitions from a saddle point to a minimum. 
  Numerical simulations show that using a correlation loss and a mild overparameterization suffices to learn all of the index vectors when they are nearly orthogonal, however, the correlation loss fails when the dot product between the index vectors exceeds a certain threshold. 
\end{abstract}

\section{Introduction}

Suppose a neural network model is trained to approximate the input-output pairs generated by a multi-index function by following the gradients of the loss function. 
Such loss functions are notoriously non-convex and complex to analyze in general due to the commonly observed degeneracy at initialization and a large number of fixed points, including local minima.
As a result, a randomly initialized algorithm may take a long time to find the subspace of the index vectors, and even then, may fail to match the index vectors with network neurons. 
The failure of the algorithm may be due to poor initialization--which can be overcome using overparameterization--or due to the emergence of local minima related to the geometry of the index vectors. 

More concretely, we study a stylized setting where the multi-index function is $\smash{f^*(\bx) = \sum_{j=1}^k \sigma^*(\bv_j^T \bx)}$ for unit-norm index vectors $\bv_1,\dots,\bv_k$, while the neural network is $\smash{f(\bx) = \sum_{i=1}^n \sigma(\bw_i^T \bx)}$ for unit-norm so-called neurons (that are also vectors) $\bw_1,\dots,\bw_n$.
The activation functions $\sigma^* \!,\! \sigma \!:\! \R \to \R$ may differ but are fixed in both networks.
We use the correlation loss as the training objective
\begin{align}\label{prob:multi-neuron}
    \min C - \E_{\bx \sim \N(0, I_d)} \biggl[ \sum_{i=1}^n \sigma (\bw_i^T \bx) \sum_{j=1}^k \sigma^*(\bv_j^T \bx) \biggr]
\end{align}
where the expectation is taken over the standard Gaussian distribution $\N(0,I_d)$ in $\R^d$. 
Some variations of this setting are extensively used to study the ability of neural networks to learn a low-dimensional subspace in high-dimensions \citep{ba2022high, bietti2022learning, barak2022hidden, damian2022neural, abbe2023sgd, berthier2023learning, bietti2023learning} but using modified algorithms that may or may not be related to the behavior of gradient flow.
Exciting recent work \citep{glasgow2023sgd, oko2024learning, arous2024high} proved global convergence results for gradient flow and characterized the end-to-end behavior of individual neurons. 
Some of these works and many others have provided sample complexity guarantees in the scenario when gradient flow is successful in the population loss limit. 
The question of optimal sample complexity in such scenarios is an active area of research and requires a careful selection of learning rate and whether to reuse the data to mitigate the noise due to finite sampling \citep{damian2023smoothing, dandi2024benefits, lee2024neural, arnaboldi2024repetita}.

In this work, we study the gradient flow of the population loss and characterize whether it succeeds, or fails. 
This question is studied by \cite{safran2018spurious} through high precision numerical simulations when $\sigma(x) \!=\! \sigma^*(x) \!=\! \max(0, x)$, showing that gradient flow fails when $n \!=\! k \!\geq\! 6$ under the mean squared error (MSE) loss. 
\cite{arjevani2021analytic, simsek2023should} gave algebraic expressions for the fixed points associated with partitions of the symmetry group for the ReLU and erf activation functions, respectively. 
The combinatorial growth of the number of fixed points for the MSE loss \citep{simsek2021geometry} makes it very challenging to study the end-to-end gradient flow dynamics for finite-width neural networks. 
Indeed, for the more realistic setting of the MSE loss, this is only achieved so far for a single-index model \citep{xu2023over} thanks to the absence of the fixed points emerging from the combinations of the index vectors, or for quadratic activation function \citep{martin2024impact} which trades the permutation symmetry with rotational symmetry.
In this paper, following \cite{bietti2023learning, arous2024high}, we study the correlation loss instead of the MSE loss and consider a broad family of activation functions. 
A key difference is that we do not restrict the weight space to an orthogonal frame (aka the Stiefel manifold), which allows for studying the success or failure of gradient flow when learning index vectors in more general positions than an orthogonal frame. 

More concretely, the model~\eqref{prob:multi-neuron} features a perfect decoupling of neural network neurons. 
Thus, the multi-neuron problem reduces to studying the trajectory of a single neuron given the initial condition, and how good the random initialization over multiple neurons is. 
The evolution of a single neuron is described by the following dynamical system
\begin{align}\label{prob:single-neuron}
    \frac{d}{dt} \bw_i &= -(I_d - \bw_i \bw_i^T) \nabla L(\bw_i) \notag \\
    L(\bw_i) &= C - \E \biggl[ \sigma (\bw_i^T \bx) \sum_{j=1}^k \sigma^*(\bv_j^T \bx) \biggr] 
\end{align}
where $\nabla$ denotes the standard Euclidean gradient, $(I_d - \bw_i \bw_i^T)$ is the projection onto the tangent space of the unit sphere at $\bw_i$, and expectation is always taken with respect to the standard Gaussian input hence $\mathcal{N}(0, I_d)$ notation is dropped hereafter. 
\cite{mondelli2019connection} studied the same model and argued that matching the neurons to the index vectors is computationally hard when the number of index vectors exceeds $d^{3/2}$ under certain complexity-theoretic assumptions on tensor decomposition. 
We take a purely geometric approach and prove that a special fixed point---one that computes the average of the index vectors---transitions from a saddle point to a minimum when the dot product between distinct index vectors exceeds an explicit threshold. 
Overall our analysis shows a dichotomy in gradient flow behavior: when the index vectors are orthogonal, we prove that the single neuron converges to the nearest index vector, whereas, when the index vectors are too close or too many with small positive dot products, the single neuron converges to the average of the index vectors, hence failing to match the neurons to the index vectors.
We partially explain the latter behavior by proving the saddle-to-minimum transition and providing complementary numerical simulations. 

Our particular contributions are listed below:

\begin{itemize}
    \item while our time complexity analysis is similar to recent work following \cite{arous2021online} when considering population gradient flow, it applies to the most general setting when the index vectors are in arbitrary directions (Section~\ref{sec:time-complexity});
    \item to show convergence to the nearest index vector, we identify a Lyapunov function that has a monotonic behavior over time (Section~\ref{sec:dir-convergence}), as a result, a mild overparameterization of $k \log (k)$ neurons is sufficient for learning all of the index vectors with gradient flow, (Section~\ref{sec:gflow-mild-OP}); here we consider random initialization of neurons on the sphere, and these two results give use the complete qualitative characterization of the dynamical system;
    \item we then focus on the effect of the geometry of index vectors, and prove a saddle-to-minimum transition when the index vectors form an equiangular frame (Section~\ref{sec:saddle-to-min});
    \item numerical simulations show that the gradient flow gets stuck at this minimum and fails to match the neurons to the index vectors when the index vectors get too close to each other (Section~\ref{sec:simulations});
    \item finally, we establish an exact correspondence between tensor decomposition and the neural network model~\eqref{prob:multi-neuron} which shows the fundamental difficulty of studying gradient flow for arbitrary geometries since most tensor problems are NP hard \citep{hillar2013most}.
\end{itemize}

\subsection{Notation}

We use bold letters to denote vectors such as $\bw, \bv, \bx$ and capital letters for the matrices such as 
$W \!=\! [\bw_1, \dots, \bw_n]$ which has size $ d \times n$ and $V \!=\! [\bv_1, \dots, \bv_k]$ which has size $d \times k$. 
The matrix of dot products between the index vectors is denoted by $A \!=\! V^T V$ which has size $k \times k$.
We assume that the index vectors are linearly independent, hence $A$ is invertible.
The set of indices is denoted by $[k] \!=\! \{ 1, \dots, k \}$.  
Since the input data distribution is standard Gaussian, it is natural to expand the activation functions $\sigma, \sigma^* \!:\! \R \!\to\! \R$ using Hermite polynomials.
Specifically, the following inner product
\[    
    \bigl\langle f_1, f_2 \bigr\rangle_\phi \!:=\! \int_{\R} f_1(x) f_2(x) \phi(x) dx, \quad \phi(x) \!=\! \frac{1}{\sqrt{2 \pi}} e^{-x^2/2} 
\]
is used to apply a Gram-Schmidt process to the sequence of monomials $1, x, x^2, x^3, \dots $ which gives the probabilist's Hermite polynomials. 
It is convenient to use the normalized sequence 
\[
    h_0 \!=\! 1, \ \ h_1(x) \!=\! x, \ \ h_2(x) \!=\! \frac{x^2 \!-\! 1}{\sqrt{2}}, \ \ h_3(x) \!=\! \frac{x^3 \!-\! 3x}{\sqrt{3!}}, \dots
\]
which ensures $ \bigl\langle h_p, h_{p'} \bigr\rangle_{\phi} \!=\! \delta_{p p'}$. 
We assume that the activation functions have finite norm $\| \sigma \|_{\phi}, \| \sigma^* \|_{\phi} \!<\! \infty$ where the norm here is defined with respect to the inner product $ \| \sigma \|_{\phi} \!=\! \bigl \langle \sigma, \sigma \bigr \rangle_{\phi}^{1/2} $.
The activation functions $\sigma, \sigma^* \!:\! \R \!\to\! \R$ are then expanded as
\[
    \sigma(x) = \sum_{p=1}^\infty a_p h_p(x), \quad \quad \sigma^*(x) = \sum_{p=1}^\infty b_p h_p(x),
\]
where the convergence holds with respect to the norm $\| \cdot \|_\phi$. 
The sequences $(a_p)_{p \geq 1}, (b_p)_{p \geq 1}$ are called the Hermite coefficients. 
For inner products with respect to two-dimensional Gaussian distribution, we use expressions involving the Hermite coefficients. 
For the Euclidean $\ell_p$ norm of vectors, we use $\| \cdot \|_p$. 

\subsection{Low Dimensional ODE}

The ODE~\eqref{prob:single-neuron} describes the movement of $d$-dimensional vector $\bw$ (also called a neuron) on the unit sphere. 
Due to the rotational invariance of the standard Gaussian, the two-point Gaussian integral can be expressed as 
\[
    g_{\sigma, \sigma^*} (\bw^T \bv_j) \!=\! \E[\sigma(\bw^T \bx) \sigma^*(\bv_j^T \bx)] 
\]
where $g_{\sigma, \sigma^*} \!:\! [-1,1] \to \R$ is the natural extension of the so-called dual activation function \citep{daniely2016toward} to the case when $\sigma$ and $\sigma^*$ may be non-matching. 
As a result, one can express the correlation loss in terms of the dot products between $\bw$ and $\bv_1, \dots, \bv_k$. 
This enables us to express the surface of the loss $L : \mathbb{S}^{d-1} \to \R$ using $k$ variables 
\begin{align*}
    L_0(V^T \bw) &= L(\bw).
\end{align*}
The loss $L_0 : \R^k \to \R$ can be viewed as mapping the subspace spanned by the index vectors to the loss values. 
In effect, it will be sufficient to analyze the behavior of the induced flow on this subspace.

\begin{lemma}\label{lem:projection} Assume that $\bw(t)$ solves the ODE ~\eqref{prob:single-neuron}  
given an initial condition $\bw \!=\! \bw_0$. 
Then the vector of dot products $\bu(t) = V^T \bw(t)$ solves the following ODE
\[
     \frac{d}{dt} \bu = - (A - \bu \bu^T) \nabla L_0(\bu) 
\]
with an initial condition $ \bu_0 = V^T \bw_0 $.  
\end{lemma}

The proof uses a simple chain rule (see Appendix~\ref{sec:reparam-loss-app}). 
Note that when the index vectors are orthogonal, the induced flow is a spherical gradient flow. 

Using a nice property of the Hermite polynomials \citep[Chapter 11.2]{o2021analysis}, one can expand the two-point Gaussian integral, hence the loss $L_0$ is 
\begin{align}\label{eq:corr-loss}
    L_0(\bu) = C - \sum_{p = 1}^\infty a_p b_p \| \bu \|_p^p
\end{align}
where $\bu = V^T \bw$. 
The domain of the loss $L_0$ is determined by the geometry of the index vectors and is an ellipsoidal ball $D = \{ \bu : \bu^T A^{-1} \bu \leq 1 \}$. 
For orthogonal index vectors, note that the domain is a unit ball since $A = I_k$. 
The derivation of the boundary uses simple linear algebra and is given in Appendix~\ref{sec:reparam-loss-app}.

This paper considers the case when $a_p b_p = 0$ for $p = 1$ which implies that $\bu = 0$ is a fixed point of the loss since $\nabla L_0(\bu) \bigl|_{\bu = 0} = 0$.  
More precisely, the lowest degree Hermite polynomial in the expansion determines the degree of degeneracy of the fixed point $\bu = 0$.
This important quantity of the target activation function $\sigma^* : \R \to \R$, that is,  
\[
    p^* := \text{argmin}_{p \geq 1} \left( b_p \neq 0 \right). 
\]
is called the information exponent in the literature \citep{dudeja2018learning, arous2021online}.

\begin{assumption}\label{ass:inf-exp} We consider the case when $p^* \geq 3$, or $p^* \geq 2$ and there is a higher-order component breaking the rotational symmetry, i.e., $ a_p b_p \neq 0$ for some $p \geq 3$. 
\end{assumption}

It is necessary to assume $(a_p b_p) \bigl|_{p = p^*} > 0$ to ensure the decay of the loss near $ \bu = 0$ at initialization. 
We also assume the following throughout the paper for technical reasons.

\begin{assumption}\label{ass:hermite-coeff} The Hermite coefficients of the activation functions $\sigma, \sigma^*$ have the same sign:
\[
    c_p := a_p b_p \geq 0 \quad \text{for all} \quad p \geq p^*.
\]
\end{assumption}
Assumption~\ref{ass:hermite-coeff} is satisfied for example for the matching activation functions $\sigma \!=\! \sigma^*$. 
It may be possible to relax this Assumption to allow some negative coefficients in the higher-order terms as long as there are positive coefficients with larger magnitudes in the lower-order terms.

\begin{assumption}\label{ass:series-convergence}
We assume that the following series converges
\[
\sum_{p \geq p^*}a_pb_pp \lambda_{\max}(A)^{\frac{p}{2}} < \infty.
\]
\end{assumption}

Assumption~\ref{ass:series-convergence} ensures that the gradient of the correlation loss is well-defined. 

Assumption~\ref{ass:inf-exp}, Assumption~\ref{ass:hermite-coeff}, Assumption~\ref{ass:series-convergence} are standard in literature and used everywhere.

\begin{assumption}\label{ass:index-vectors} We assume that the index vectors have a positive dot product
\[
    \bv_j^T \bv_{j'} \geq 0 \quad \text{for all} \quad j,j' \in [k].
\]
\end{assumption}

We need this assumption to avoid situations such as the following worst case, $\bv_j \!=\! -\bv_{j'}$ for a pair of $j, j'$ where the two index vectors would be pulling the single neuron in the opposite directions (that is, for non-even activations). This scenario would require a finer analysis than what we do. 
In particular, for odd activations, the gradient corresponding to index vector $j$ and $j'$ cancels out hence learning is impossible. 

Assumption~\ref{ass:index-vectors} is also used everywhere.
This assumption is sufficient for characterizing the time complexity of the gradient flow at early times for arbitrary geometries of the index vectors that satisfy it. 
However, the fixed point structure and infinite-time behavior of gradient flow is tightly coupled with the exact geometry of the index vectors, hence, we need a more rigid assumption to study these questions. 

\begin{assumption}\label{ass:equi-frame} We assume that the index vectors for an equi-angular tight frame with a non-negative dot product
\[
    \bv_i^T \bv_j = \beta \in [0, 1]
\]
where $\beta=0$ corresponds to the orthogonal frame, and $\beta=1$ corresponds to all index vectors being equal to each other, resulting in a single-index model of the form $f^*(\bx) = k \sigma^*(\bv_1^T \bx)$. 
\end{assumption}

\section{Learning Index Vectors}

\subsection{Time Complexity}
\label{sec:time-complexity}

A learning algorithm should ensure a non-vanishing correlation between the predictor $f$ and the target $f^*$:
\[
    \E [f(x) f^*(x)] > \text{const.}
\]
which is vanishing at random initialization due to high dimensionality of the unit sphere ($d$ is large). 
In Subsection~\ref{sec:time-complexity}, we study how much time gradient flow spends to learn a non-vanishing correlation with the target function. Since the dynamics of individual neurons are decoupled, it suffices to study the behavior of a single neuron. 

A vector $\bw$ (neuron) is sampled uniformly on the unit sphere at initialization. 
The dot product between this vector and an unknown vector is small: with high probability, $\bw(0)^T \bv_0 = \Theta (d^{-1/2})$. 

Following gradient flow, the vector will aggregate the spherical gradient of the correlation loss
\[
    \nabla^S L(\bw) = \sum_{p \geq p^*} a_p b_p p \bigl(\sum_{j=1}^k (\bv_j^T \bw)^{p-1} \bv_j -   \sum_{j=1}^k (\bv_j^T \bw)^{p} \bw \bigr)
\]
which points to the $k$-dimensional subspace spanned by the index vectors at early times since the dot products $\bw(0)^T \bv_j$ are small for all $j \in [k]$.  
In Theorem~\ref{thm:time-complexity}, we answer the question: how much time $T$ is necessary and sufficient to find the subspace, i.e., to reach non-vanishing projection $\| V^T \bw(T) \|_2 = \Theta(1)$? 

\begin{theorem}[Time complexity]\label{thm:time-complexity} We assume that $d$ is large and $k \ll d$. 
Sample $\bw(0)$ uniformly on the sphere such that $\bw(0)^T \bv_j = \Theta(d^{-1/2})$ with high probability. 
We also assume that $\bw(0)^T \bv_j>0$ for all $ j \in [k]$.
For $p^* \geq 2$, we show the following time complexities 
\begin{itemize}
    \item if $p^* \geq 3$, $T=\Theta(d^{p^*/2-1})$ is necessary and sufficient to reach $\| V^T \bw(T) \|_2 = \Theta(1)$, 
    \item if $p^* = 2$, $T=\Theta(\log(d))$ is necessary and sufficient to reach $\| V^T \bw(T) \|_2 = \Theta(1)$.
\end{itemize}
\end{theorem}
  
\begin{proof}[Proof sketch]
The key step in the proof is confining the time evolution of $S = \sqrt{\bu^T A^{-1} \bu}$ which is the generalized Euclidean norm of $\bu$ with respect to the metric $A^{-1}$. 
In effect, we bound the growth of the quantity $S$, which in turn gives the growth of $\ell_2$ norm of $\bu$ thanks to the following bounds 
\[
    \lambda_{\min}(A)^{1/2} S \leq \| \bu \|_2 \leq  \lambda_{\max}(A)^{1/2} S
\]
where $\lambda_{\min}(A)$ and $\lambda_{\max}(A)$ are the smallest and biggest eigenvalues of $A$ respectively. 
(We regard $\lambda_{\min}(A)$ and $\lambda_{\max}(A)$ as constants here since $k \ll d$.)
Explicitly writing the time derivative of $S$ gives an expression including higher-order powers of $\bu$. 
Using Hölder type inequalities, we give upper and lower bounds for each such term in terms of $S$. 
This effectively reduces the system to one dimension, hence $S$ has the same time complexity as the single index case which completes the proof.
\end{proof}

In qualitative terms, the initial escape time is determined by the lowest-order term of the correlation loss: a bigger lowest-order term resulted in gradients with a smaller norm, hence increasing the time complexity of the gradient flow.
Theorem~\ref{thm:time-complexity} gives a sharp characterization of the time complexity (matching upper and lower bounds). 
The proof is given in Appendix~\ref{sec:time-complexity-app}.

\begin{remark} \cite{arous2021online} obtains the same time complexity guarantee $\Theta(d^{p^*/2-1})$ for the single index model. 
More recently, \cite{arous2024high} obtained the lower bound on the time complexity $\Omega(d^{p^*/2-1})$ for a multi-index model for the special case when the index vectors are orthogonal.
Our Theorem~\ref{thm:time-complexity} generalizes these results for arbitrary geometries of index vectors and provides both the upper and lower bounds on time complexity of gradient flow. 
\end{remark}

\paragraph{Further Generalizations} In Theorem~\ref{thm:time-complexity}, we assumed that $k \ll d$ which implied that the magnitude of the projection is $\| V^T \bw(0) \|_2 = \Theta (d^{-1/2})$.
In principle, our proof technique allows us to study the case $k = d^{\gamma}$ where $\gamma \in (0, 1)$ and $d$ is large. 
First, at initialization, the norm of the projection to the subspace is $ \| V^T \bw(0) \|_2 = \Theta (d^{-(1-\gamma)/2})$ with high probability. 
If the index vectors are orthogonal, it is possible to generalize our time complexity guarantees by recalculating the definite integrals with the new initial condition. 
If the index vectors are in arbitrary positions (e.g., sampled uniformly from the Haar measure on the unit sphere), $\lambda_{\min}(A)$ and $\lambda_{\max}(A)$ may depend on $d$ and change the time complexity in a non-trivial way.
Studying this setting would require tools from random matrix theory and is beyond the scope of this paper. 

\subsection{Directional Convergence}
\label{sec:dir-convergence}

In Subsection~\ref{sec:dir-convergence}, we study the infinite-time behavior of gradient flow when the index vectors form an orthonormal frame. 
In particular, we characterize the complete set of fixed points by establishing an exact correspondence with tensor eigenvectors (Proposition~\ref{prop:fixed-points}) and give global convergence guarantees (Proposition~\ref{prop:dir-converge}). 
Our analysis fully characterizes the qualitative behavior of a single neuron. 
We only require $d \geq k$ for the orthogonal frame in this Subsection. 

\begin{proposition}[Fixed Points $\leftrightarrow$ Eigenvectors]\label{prop:fixed-points} Assume that $\sigma^* = h_{p^*}$, $p^* \geq 3$,  and $\bv_i^T \bv_j = \delta_{ij}$.
The fixed points of the dynamical system~\eqref{prob:single-neuron} are in one-to-one correspondence with the unit eigenvectors of the $p^*$-th order tensor $T = \sum_{j=1}^k \bv_j^{\otimes p^*}$.    
\end{proposition}

\begin{proof} Let $\bw$ be a unit vector. 
It is an eigenvector of $T$ if and only if 
\[
    T(\bw, \dots, \bw, I) = \sum_{j=1}^k (\bv_j^T \bw)^{p^*-1} \bv_j = \lambda \bw
\]
where the first equality holds due to the orthogonality of $\bv_j$.
On the other hand, $\bw$ is a fixed point of the dynamical system~\eqref{prob:single-neuron} when the Euclidean gradient is orthogonal to the tangent subspace of the unit sphere, that is 
\begin{align*}
    \nabla L(\bw) = \lambda \bw \ \ \Leftrightarrow \ \ \sum_{j=1}^k (\bv_j^T \bw)^{p^*-1} \bv_j = \lambda \bw.
\end{align*}
Note that the two conditions are the same. 
\end{proof}

\paragraph{Complete set of fixed points} Proposition~\ref{prop:fixed-points} shows that the fixed points of the dynamical system are in one-to-one correspondence with the eigenvectors of the tensor $T$ in the orthonormal case. 
Fortunately, \citet[Theorem 2.3]{robeva2016orthogonal} characterized the complete set of eigenvectors for the orthonormal tensor, which are (informally stated)
\begin{align*}
    &\sum_{j \in I} \xi_j \bv_j, \quad I \neq \emptyset, I \subset [k], \ \xi_j \in \{ (-1)^{p^*+1}, 1 \}, \ \text{and} \\
    &\bv^\perp \ \ \text{such that} \ (\bv^\perp)^T \bv_j = 0 \  \text{for all} \ j \in [k].
\end{align*}
Proposition~\ref{prop:fixed-points} gives us that normalizing the eigenvectors to unit norm gives the complete set of fixed points of the dynamical system~\eqref{prob:single-neuron} when the target activation function has a single Hermite component.
We call $\{\xi_1 \bv_1, \dots,  \xi_k \bv_k \} $ `pure' fixed points, and the remaining ones `mixed' fixed points which compute the normalized average of more than one index vector. Geometrically, this gives us (for $p^*$ odd): $ 2^k -1 $ isolated fixed points and a subspace of fixed points that is $d-k$ which is an empty set for $d=k$.

\begin{figure}[t]
     \centering
     \includegraphics[width=0.5\textwidth]{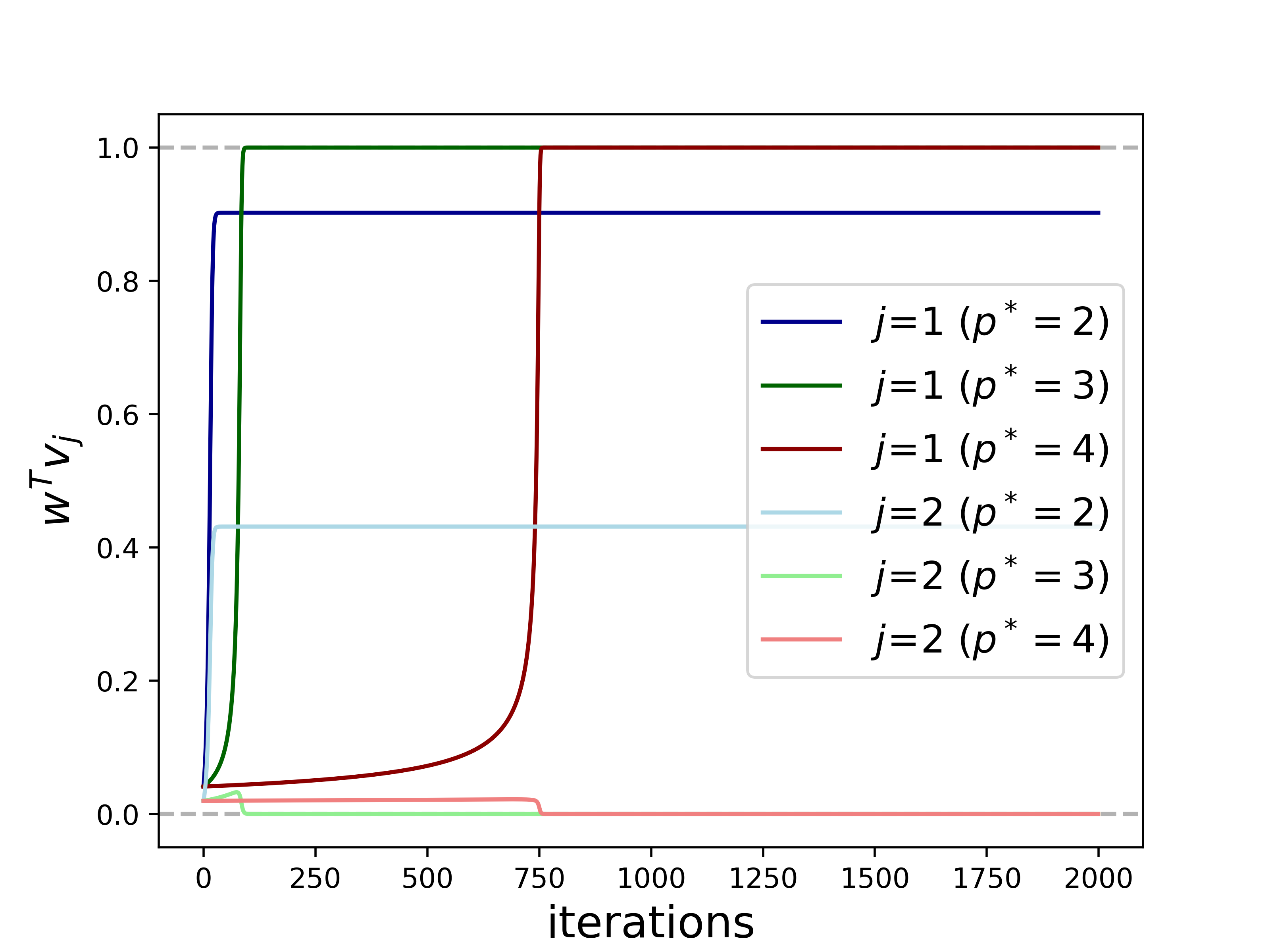}
     \caption{
     \textit{Dot products during training; $n=1, d=1000$.} 
    We run gradient descent by updating the unit norm vector with the spherical gradient using a learning rate $\eta=0.1$ and normalizing the vector after the update. 
    Here $\sigma^* = h_{p^*}$ and $k=2$. 
    For $p^* \! \in \! \{3,4\}$, the unit vector (neuron) converges in the direction of the nearest index vector at initialization ($j \!=\! 1$). 
    For $p^* \!=\! 2$, a linear combination of the two directions is learned due to rotational symmetry. 
    The maximum dot product reaches a constant value in a longer timescale when the information exponent is bigger.}
    \label{fig:n1-simulations}
\end{figure}

\begin{proposition}[Directional Convergence]\label{prop:dir-converge} 
Let $\ell \in [k]$ and $I = [k] \setminus [\ell] $. 
Assume that $\bw(0)^T \bv_j > 0$ for all $j \in [k]$ and that these dot products are ordered as $\bw(0)^T \bv_1 = \dotsb = \bw(0)^T \bv_\ell > \max_{j \in I} \bw(0)^T \bv_j $ wlog. 
For $\sigma^* = h_{p^*}$, $p^* \geq 3$, one gets
\[
    \lim_{t \to \infty} \bw(t) \!=\! \frac{1}{\sqrt{\ell}} \sum_{j=1}^\ell \bv_j.
\]
Furthermore, if $\sigma^*$ is even, one can drop the assumption that the dot products are positive at initialization. 
Assume that the magnitudes of the dot products are ordered as $|\bw(0)^T \bv_1| = \dotsb = |\bw(0)^T \bv_\ell| > \max_{j \in I} |\bw(0)^T \bv_j|$ wlog. In this case, one gets 
\[
    \lim_{t \to \infty} \bw(t) = \frac{1}{\sqrt{\ell}} \sum_{j=1}^\ell \operatorname{sgn}(\bw(0)^T \bv_j) \bv_j.  
\]
\end{proposition}

A random initialization breaks ties between any pair of dot products almost surely, hence, we conclude that the single neuron converges to the nearest index vector. 

\begin{proof}[Proof sketch] Let us sketch the proof for the case of general $\sigma^*$ and positive dot products at initialization. 
In a nutshell, we identify the following Lyapunov functions
\[
    \Delta_{1j}(t) = \bw(t)^T \bv_1 - \bw(t)^T \bv_j \quad \text{for} \ \ j \neq 1
\]
and show that
\begin{align}\label{eq:invariance}
    \frac{d}{dt} \Delta_{1j}(t) \geq 0 \quad \text{given} \quad \Delta_{1j}(0) > 0 
\end{align}
using some inequalities. 
This invariance~\ref{eq:invariance} implies that $\Delta_{1j}(t) = \bw(t)^T \bv_1 - \bw(t)^T \bv_j > 0$ for all times. 
The only fixed point breaking ties between $\bv_1$ and $\bv_j$'s, hence, satisfying the above condition is $\bw = \bv_1$.
We conclude convergence to the fixed point by the stable manifold theorem. 
Furthermore, one can notice the following 
\[
    \Delta_{1j}(t) = 0, \quad \text{given} \quad \Delta_{1j}(0) = 0
\]
due to the symmetries. 
This implies convergence to a mixed fixed point, which computes a normalized average of the index vectors that all received the equal maximum dot product at initialization. 
\end{proof}

The proof features that $\bw(t)^T (\bv_1 - \bv_j)$ is non-decreasing over time, ensuring that $\bw$ moves gradually toward $\bv_1$. 
The full proof is given in the Appendix~\ref{sec:directional-convergence-app}.

\begin{remark}[$p^*=2$]\label{rmk:rotational-sym} Proposition~\ref{prop:dir-converge} also applies when $p^* = 2$ as long as there is another Hermite polynomial in the expansion without a non-zero coefficient, i.e., $a_p b_p > 0$ for some $p > 2$. 
The special case $\sigma^*(x) = h_2(x)$ generates the loss function $L(\bw) = C - a_2 b_2 \sum_{j=1}^k (\bw^T \bv_j)^2 = C - a_2 b_2 \| V^T \bw \|_2^2 $. 
In this case, any $\bw = V \bu$ for some $\| \bu \| = 1$ has a unit norm and minimizes the loss function.  
\end{remark}

\begin{remark}[$p^*=1$]\label{rmk:linear-comp}
When the target activation function has a linear component, i.e., $p^*=1$, then $\bv^\perp$ is not a fixed point, hence, the gradient flow initialized randomly on the sphere does not suffer from the time complexity of escaping a fixed point. 
Moreover, the monotonic behavior of the Lyapunov function is no longer guaranteed in this setting, preventing us from characterizing infinite-time behavior. 
Finally, the correspondence with the tensor eigenvector problem is no longer valid, hence, preventing us from characterizing the fixed points with the same method. 
\end{remark}

\paragraph{Comparison to Tensor Decomposition}
Proposition~\ref{prop:fixed-points} established an exact correspondence between the fixed points of the dynamical system and the eigenvectors of the corresponding tensor for the orthonormal frame, for the case $\sigma^* \!=\! h_{p^*}$.
We now informally compare the behavior of the gradient flow with the power iteration method. 
Recall that the $p^*$-th order tensor is given by
\[
    T \!=\! \sum_{j=1}^{k} \bv_j^{\otimes p^*}.
\]
\citet[Theorem 4.1]{anandkumar2014tensor} showed that the power iteration algorithm converges to the set $\{ \bv_1, \dots, \bv_k \}$ almost surely, escaping the other eigenvectors corresponding to the partial averages of the index vectors. 
Hence, the power iteration algorithm succeeds in finding the unique decomposition of the orthonormal tensor by finding one eigenvector, subtracting it from the tensor, and iterating this procedure until reaching zero approximation error. 
In this paper, we show that gradient flow performs similarly, where each neuron converges to the nearest index vector. 
One advantage of gradient flow is that neurons can move in parallel as opposed to finding the eigenvectors sequentially in tensor decomposition. 
However, this parallelization comes with a small cost of $\ln(k)$ factor of neurons to ensure the picking of each index vector, which is what we study in the next Subsection. 

\subsection{Mild Overparameterization Works}
\label{sec:gflow-mild-OP}

Because we study a fully decoupled model~\ref{prob:multi-neuron}, the single-neuron dynamics studied in Subsection~\ref{sec:time-complexity} and Subsection~\ref{sec:dir-convergence} fully describes the collective behavior of neurons. 
The remaining question is whether random initialization of multiple neurons ensures their directional convergence to the full set of index vectors. 

We use Proposition~\ref{prop:dir-converge} and a classic matching argument to ensure that a mild overparameterization of a logarithmic factor works. 
More precisely, it suffices to ensure that all of the $k$ directions $\bv_1,\dots,\bv_k$ are ``collected'' (in the coupon-collecting sense) by the $n$ student neurons at initialization.
Say $\bv_j$ is \emph{collected} by $\bw_i$ (at initialization) if
\begin{equation*}
    \bw_i(0)^T \bv_j > \max_{j' \neq j} \ \bw_i(0)^T \bv_{j'}.
\end{equation*}
By Proposition~\ref{prop:dir-converge}, if $\bv_j$ is collected by $\bw_i$ at initialization, then $\bw_i(t) \to \bv_j$ as $t\to\infty$.
By symmetry of the random initialization, for any $i \in [n]$ and $j \in [k]$,
\begin{equation*}
    P_{\text{collect}} := \Pr[ \text{$\bv_j$ is collected by $\bw_i$} ] = 1/k.
\end{equation*}
Therefore, by independence of the initialization and a union bound,
\begin{align}
    &\Pr[ \text{$\exists j \in [k]$ s.t.\ $\bv_j$ is not collected by any $\{\bw_i\}_{i \in [n]}$} ]
    \notag \\
    &\leq k (1-P_{\text{collect}})^n \leq k \exp(-P_{\text{collect}}n).
\end{align}
This failure probability bound is less than $1/k^{\epsilon}$ when $n \geq (1+\epsilon)k\ln k$ for any $\epsilon > 0$. 
Hence, a mild overparameterization of $\ln(k)$ factor is sufficient for matching the neurons to the index vectors when the index vectors are orthogonal to each other. 

We can also easily get the following lower-bound
\begin{align}
    &\Pr[ \text{$\exists j \in [k]$ s.t.\ $\bv_j$ is not collected by any $\{\bw_i\}_{i \in [n]}$} ]
    \notag \\
    &\! \geq (1-P_{\text{collect}})^n \to \exp(-\gamma) \ \text{as} \ k \to \infty, \ \gamma = n/k.
\end{align}
The failure probability is more than $\exp(-\gamma)$ in the proportional limit when $n = \gamma k$ for any constant $\gamma$, hence using the correlation loss does not guarantee the success of gradient flow in matching neurons to the index vectors when using a constant factor of overparameterization.
However, using only a constant factor of overparameterization seems to be sufficient when using the MSE loss (discussed informally in Section~\ref{sec:simulations}).

\section{Saddle-to-Minimum Transition}
\label{sec:saddle-to-min}

In Subsection~\ref{sec:dir-convergence}, we proved that a single neuron converges to the nearest index vector when the index vectors form an orthonormal frame. 
Generalization of this result to nearly orthogonal index vectors is established by \citet{oko2024learning} when $k$ grows together with $d$. 
However, it is not clear whether this qualitative behavior persists or undergoes a dramatic change when the index vectors approach each other arbitrarily.   

In Section~\ref{sec:saddle-to-min}, we study the effect of index vectors moving from an orthogonal frame towards each other. 
For analytic tractability, we consider the scenario in which the index vectors form an equiangular frame with a non-negative dot product, that is, $\bv_i^T \bv_j = \beta$ for all $i \neq j \in [k]$, for some $\beta \in [0, 1]$ (Assumption~\ref{ass:equi-frame}). 

Specifically, we focus on the local geometry of the point that computes the average of the index vectors 
\[
    \bar{\bw} = \frac{1}{\| \sum_{j=1}^k \bv_j \|_2} \sum_{j=1}^k \bv_j.
\]

\begin{lemma} $\bar{\bw}$ is a fixed point of the dynamical system~\eqref{prob:single-neuron} when the index vectors form an equiangular frame. 
\end{lemma}

\begin{proof} $\bw$ is a fixed point if and only if $\nabla L (\bw) = \lambda \bw $. 
Let us write the Euclidean gradient explicitly:
\[
    \nabla L (\bw) = - \sum_{p \geq p^*} a_p b_p p \sum_{j=1}^k (\bw^T \bv_j)^{p-1} \bv_j ,
\]
and plug in $\bar{\bw}$:
\[
    \nabla L (\bw) \bigl|_{\bw = \bar{\bw}} =  - \sum_{p \geq p^*} a_p b_p p \sum_{j=1}^k (\bar{\bw}^T \bv_j)^{p-1} \bv_j.
\]
Observe that $\bar{\bw}^T \bv_j $ is the same for all $j$ since the dot products between the index vectors are the same. 
Hence we can push the constant term $(\bar{\bw}^T \bv_j)^{p-1}$ outside of the summation and conclude that $\nabla L(\bar{\bw})$ is parallel to $\bar{\bw}$.
\end{proof}

Intuitively, the correlation between the predictor $\smash{f(\bx) = \sigma(\bar{\bw}^T \bx)}$ and the multi-index function increases as the index vectors approach each other, as $\bar{\bw}$ would get closer to the index vectors. 
To formalize this, we need to make $\beta$ explicit in the loss function. 
Indeed, the evaluation of the loss at $\bar{\bw}$ is given by the two instances (i) orthogonal, $\beta=0$, (ii) single-index, $\beta=1$: 
\begin{align*}
    L_\beta (\bar{\bw}) \bigl|_{\beta=0} &= C - \sum_{p \geq p^*} c_p \Bigl(\frac{1}{k}\Bigr)^{p/2} k, \\ L_\beta (\bar{\bw}) \bigl|_{\beta=1} &= C - \sum_{p \geq p^*} c_p k
\end{align*}
where $L_\beta(\bar{\bw})$ is $L(\bar{\bw})$ when the index vectors form an equiangular frame with dot product $\bv_i^T \bv_j = \beta$. 
Indeed $L_\beta (\bar{\bw}) \bigl|_{\beta=0} > L_\beta (\bar{\bw}) \bigl|_{\beta=1}$, and $L_\beta (\bar{\bw})$ decreases as $\beta$ increases (see Appendix for the calculations).

In the proportional limit $k = c d$ with $c \in (0, 1)$ and as $d \to \infty$, $L_\beta(\bar{\bw})$ approaches $C$, hence learning $\bar{\bw}$ is as bad as not learning anything. 
Whereas for $\beta = 1$, we have $\bar{\bw} = \bv_1 = \dots = \bv_k$, all fixed points collapse on each other, and $\bar{\bw}$ is the optimal. 

More generally, it is important to characterize whether $\bar{\bw}$ is a strict saddle or a local minimum. 
To do so, one needs to study how the curvature of $L_\beta(\bar{\bw})$ changes as $\beta$ increases. 
In Theorem~\ref{thm:saddle-min}, we sharply characterize the sign change in the curvature for polynomial target activation $\sigma^*$ of degree $P$ and inf. exponent $p^*$. 

\begin{theorem}[Saddle-to-Minimum]\label{thm:saddle-min} Assume $d > k$. 
$\bar{\bw}$ is a strict saddle when the dot product is upper bounded by 
\[
    \beta < \frac{p^*-2}{k+p^*-2},
\]
whereas $\bar{\bw}$ is a local minimum when the dot product is lower bounded by 
\[
    \frac{P-2}{k+P-2} < \beta. 
\]
Therefore, for the case $\sigma^* \!=\! h_{p^*}$
\[
    \beta_c = \frac{p^*-2}{k+p^*-2}
\]
is the sharp threshold characterizing the transition from a strict saddle to a local minimum. 
\end{theorem}

We now give some intuition for the saddle-to-minimum transition phenomenon.  
Each index vector attracts the single neuron towards itself. 
One can decompose this attraction gradient in terms of the attraction toward $\bar{\bw}$ and the attraction toward $\bv_j - \bar{\bw}$ for $\beta > 0$.
As $\beta$ increases, the projection of gradient onto $\bar{\bw}$ increases, hence, gradually turning it from a saddle point to a minimum. 

The full proof is given in the Appendix~\ref{sec:app-saddle-to-min}.

\cite{muller2022robust} studied the tensor decomposition problem when the component vectors form an equiangular frame, focusing on the regime when the component vectors are robust eigenvectors under the tensor power iteration algorithm. 
Our Theorem~\ref{thm:saddle-min} shows that the average fixed point turns into a local minimum when $\beta$ exceeds a certain threshold of order $1/k$. 

This suggests that there may be another threshold $\beta_f$ such that the average fixed point is not only a minimizer but also attractive under gradient flow (see Fig~\eqref{fig:bifurcation}). 
Our result indicates a hard regime for tensor decomposition as well as neural networks--which was not studied before to our knowledge.  

\begin{figure}[t]
  \centering
    \includegraphics[width=0.45\textwidth]{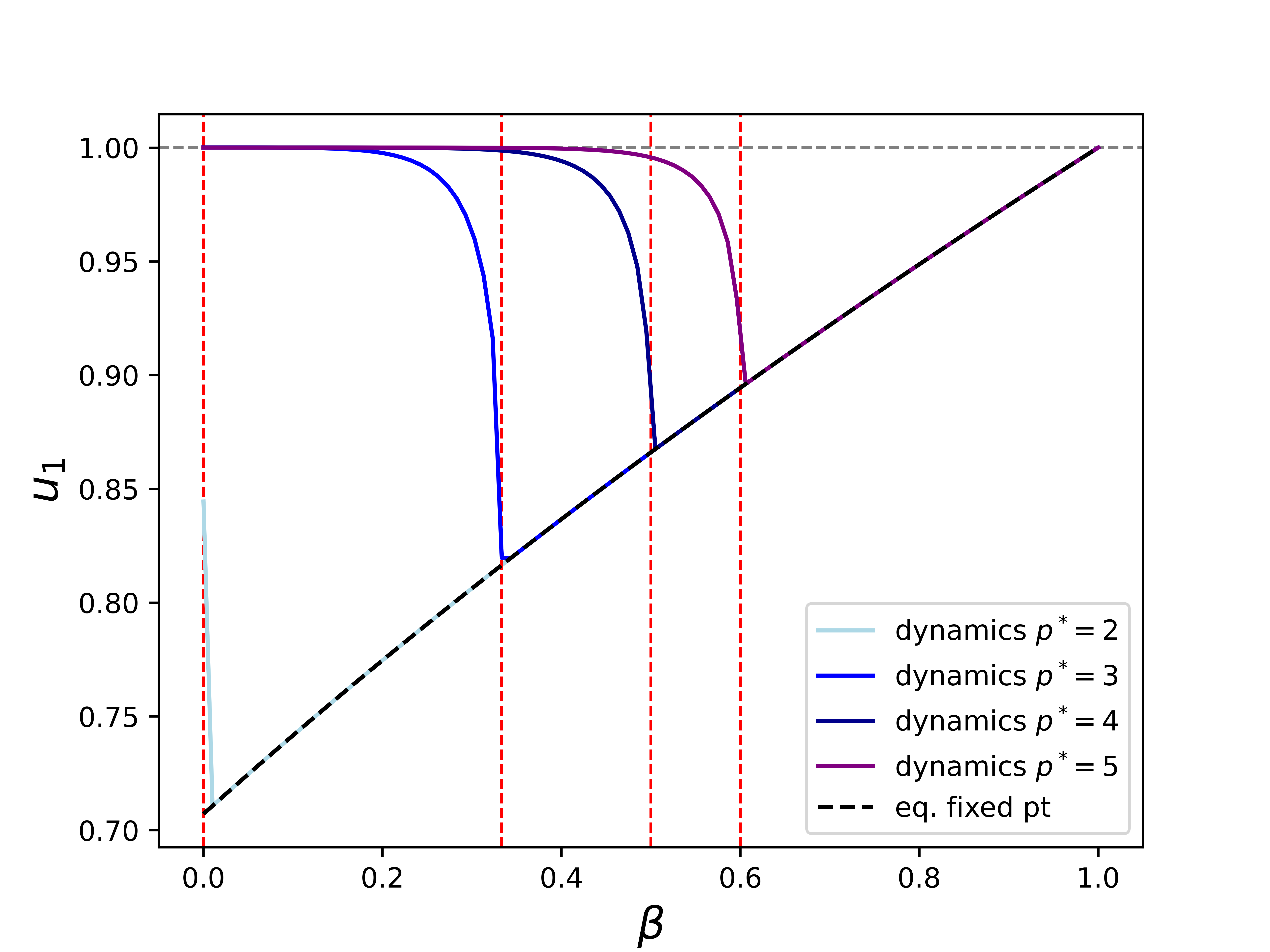}
  \centering
  \vspace{-0.25cm}
  \caption{\textit{Index vectors forming an equiangular frame with an equal dot product $\beta$, bifurcation diagram; $k=2$,  $\sigma^*=h_{p^*}$.} 
  The infinite-time behavior of the student vector abruptly changes from monotonic convergence to the nearest direction to convergence to the average of directions (non-monotonically) at a critical value $\beta_f \!\in\! (0,1)$. 
  The red dashed line indicates the saddle-to-minimum threshold $\beta_c$ given in Theorem~\ref{thm:saddle-min}. 
  Observe the small gap between $\beta_c$ and $\beta_f$. 
  } \label{fig:bifurcation}
\end{figure}

\section{Discussion}
\label{sec:discussion}

\paragraph{Limitations}
We acknowledge the following simplifications used in our model compared to the literature:
\begin{itemize}
\item The target function is assumed to be a sum of single-index models (teacher) as opposed to a general multi-index model \citep{chen2020learning, troiani2024fundamental}. 
This allows us to study the infinite-time behavior of the dynamical system when learning with gradient flow, which uses the permutation symmetry inherent in the teacher model. 
It is unclear if this type of fine end-to-end dynamical system analysis is possible for general multi-index models.  
\item The single-index components are equivalent to each other in the sense that the index vectors form an equiangular frame \text{and} the second-layer weights are fixed at one. 
Due to symmetries, the dynamical system could be reduced to a two-dimensional system, in principle. 
We argue that even the behavior can get more complicated for finite but higher-dimensional dynamical systems. 
\item We use correlation loss instead of the MSE loss. The strong interactions between the neurons when using the MSE loss make the analysis quite more complicated in the finite-width setting. 
Whether using the MSE loss brings provable improvements in time complexity and overparameterization are both intriguing questions for future work.
\item The input distribution is assumed to be standard Gaussian with identity covariance. This is a rotationally symmetric distribution where no direction is informative. Extensions to (i) covariance matrix that is identity + rank-one \citet{ba2023learning} and (ii) the uniform distribution on the sphere \citet{bruna2023single} are studied in the literature however only for single-index models. Whether such analysis could be generalized to multi-index models is an interesting avenue for future work. 
\end{itemize}

\section{Numerical Simulations}
\label{sec:simulations}

Some simulations are reported in the scenarios below:
\begin{itemize}
    \item orthogonal index vectors, compare the behavior of correlation loss and MSE loss, in terms of the benefit of MSE loss in decreasing the failure probability for insufficiently overparameterized neural networks or unlucky initializations (Fig~\eqref{fig:neuron-allocation}),
    \item orthogonal index vectors, compare the behavior of correlation loss and MSE loss, in terms of the loss curves, time complexity, and gradient flow trajectories for different activation functions (Fig~\eqref{fig:loss-curves}).
\end{itemize}

\begin{figure}[h]
  \centering
    \includegraphics[width=0.5\textwidth]{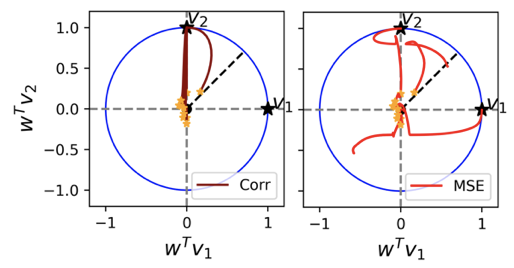}
  \centering
  \vspace{-0.5cm}
  \caption{\textit{MSE loss helps with neuron allocation; fixed initialization in both figures, $k \!=\! 2$, $n \!=\! 10$.} 
  If no neuron at initialization is closest to one of the index vectors, gradient flow fails to find it when using the correlation loss (\textit{left panel}) whereas the MSE loss fixes this issue thanks to the repulsion between neurons (\textit{right panel}).
  } \label{fig:neuron-allocation}
\end{figure}

In particular, using a factor $\gamma=4$ of overparameterization is recommendable to push the lower bound of the failure probability down to $\exp(-4) \approx 0.0183 $.
Interestingly, the factor $4$ is consistent with the numerical simulations of \cite{martinelli2023expand} for learning two-layer neural networks with gradient flow when using the MSE loss, for a large family of toy problems where the index vectors exhibit arbitrary geometries. 
In Fig.~\ref{fig:loss-curves}, we observe that MSE loss behaves similarly to correlation loss early in training when there is no linear and no quadratic component. 
Similar behavior is also observed for neural networks with ReLU activation, where early training dynamics for the MSE loss is approximated by the correlation loss \citep{maennel2018gradient, boursier2022gradient}. 

\begin{figure}[t]
     \centering
     {\includegraphics[width=0.23\textwidth]{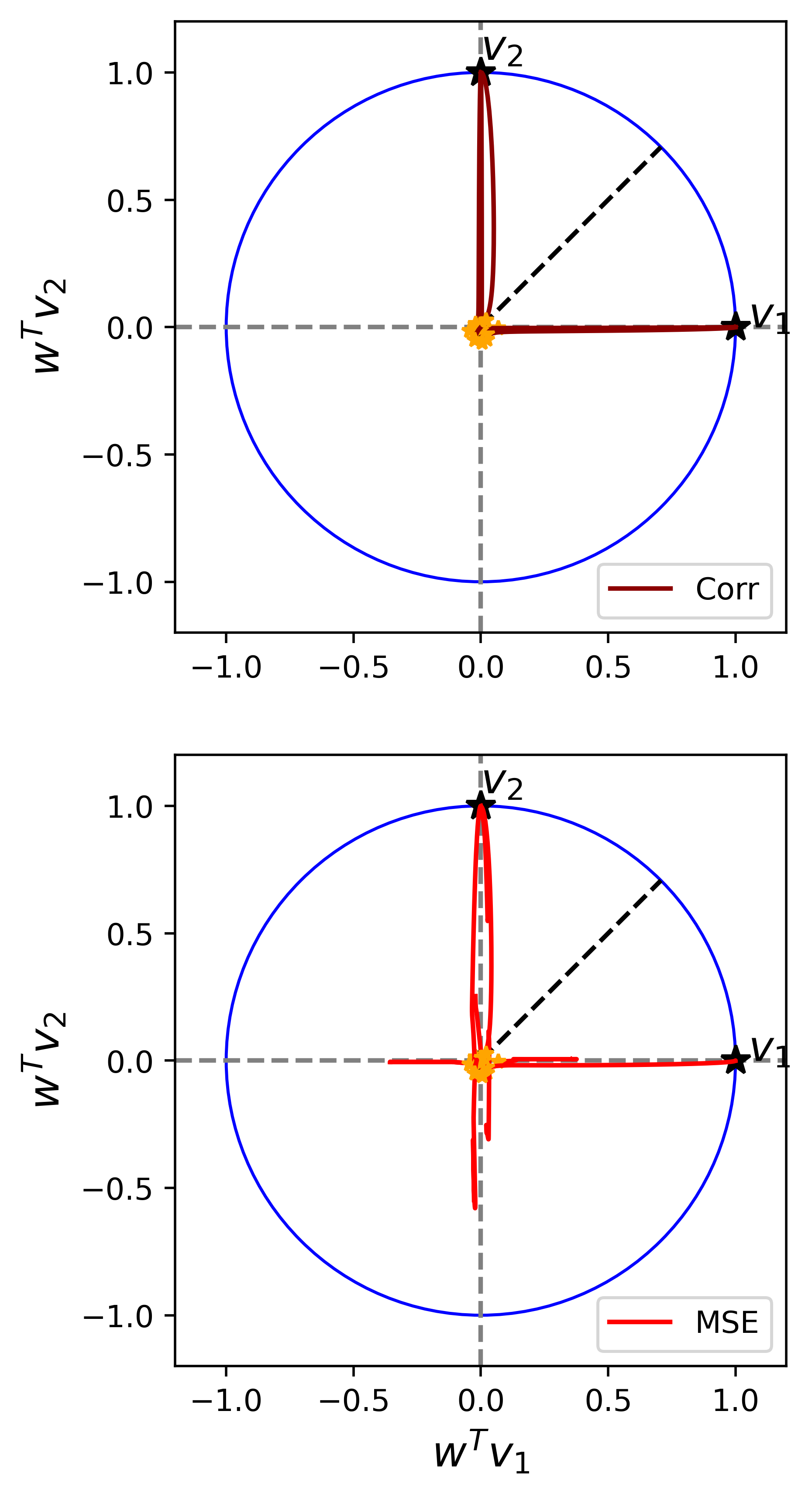}}
     {\includegraphics[width=0.23\textwidth]{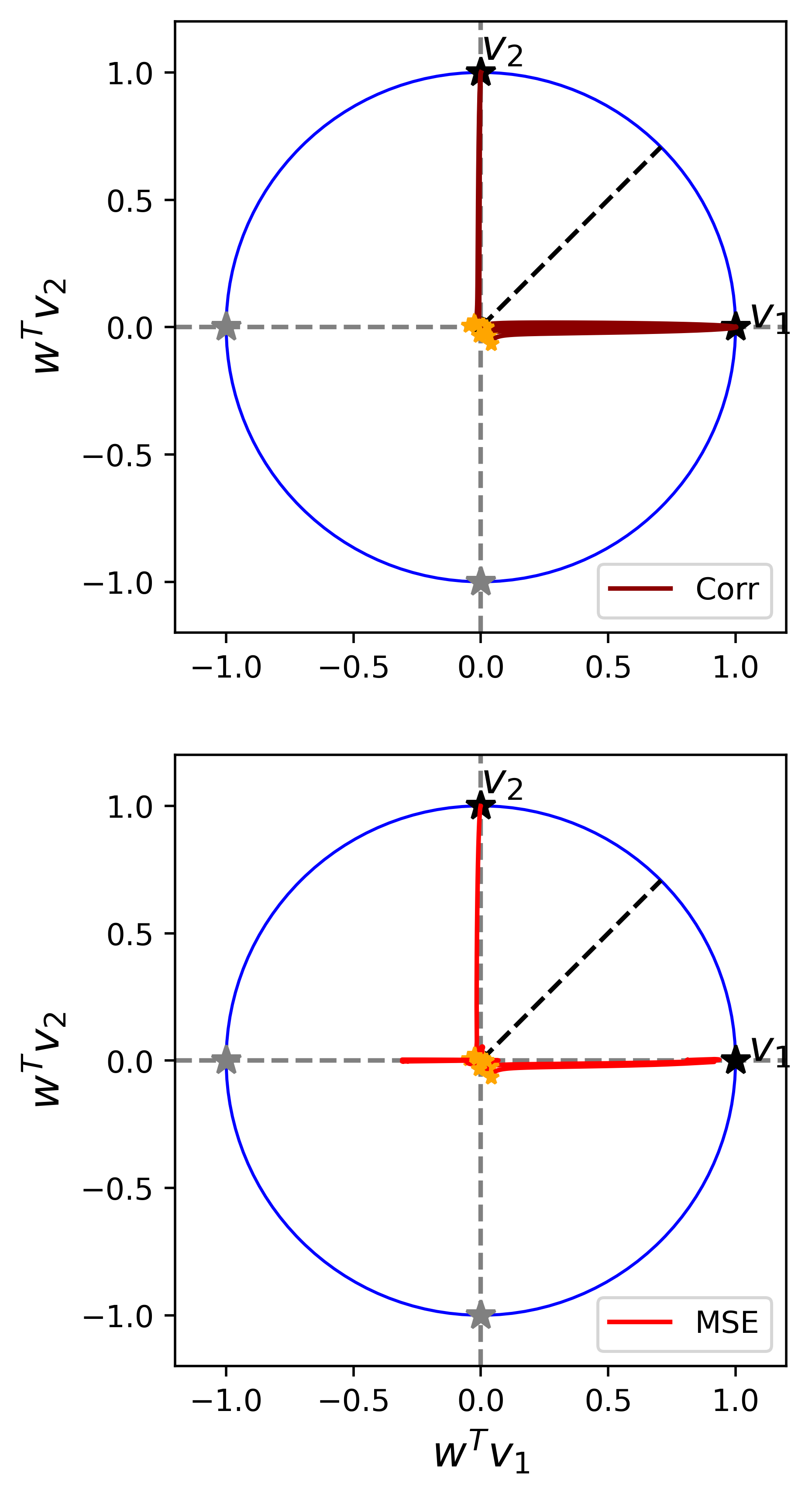}}  \\
      \vspace{-0.3cm}
     \subfloat[{\small $\sigma^* \!=\! h_3$}]
     {\includegraphics[width=0.23\textwidth]{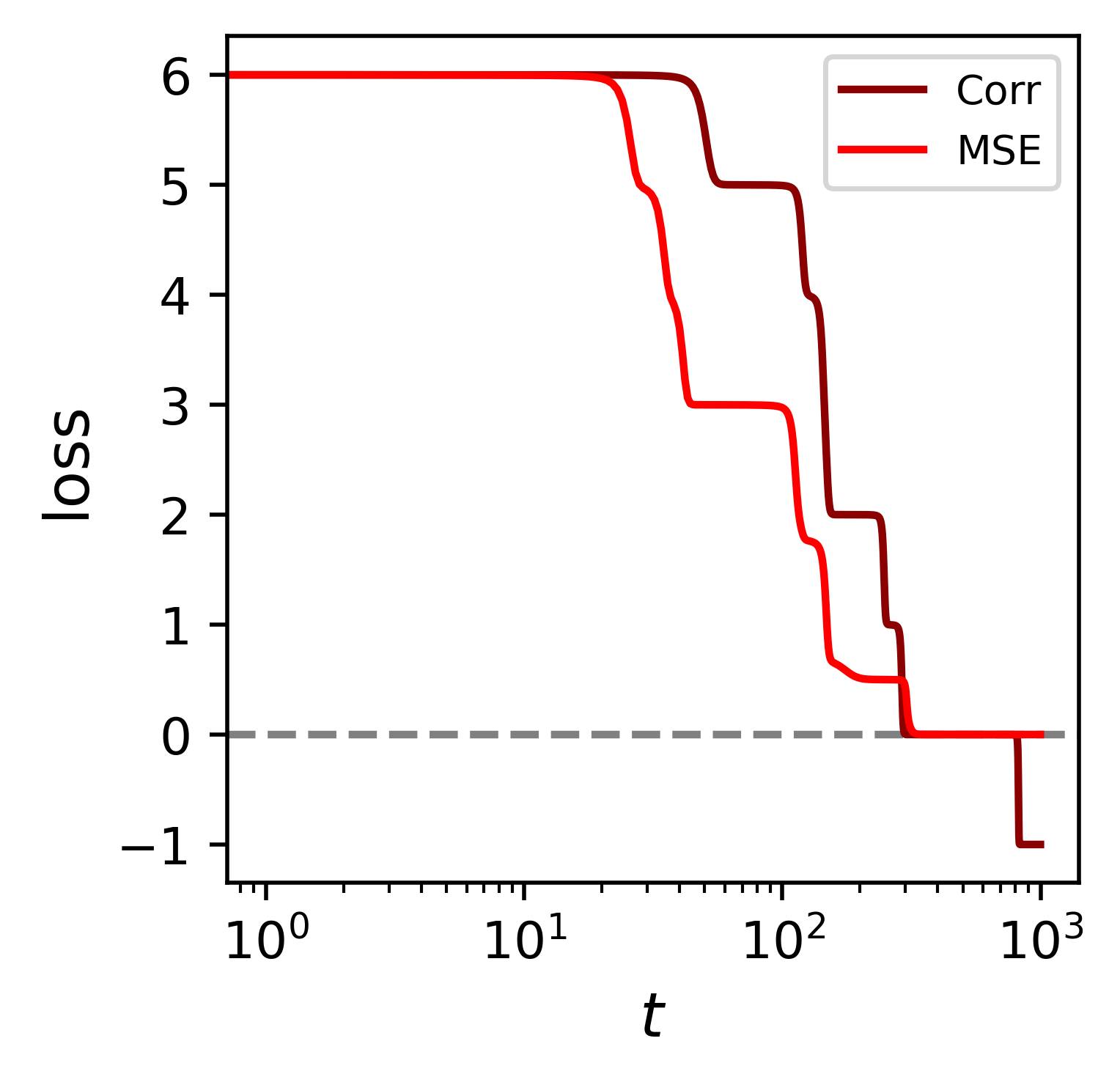}}
      \vspace{-0.3cm}
     \subfloat[{\small $\sigma^* \!=\! h_4$}]
     {\includegraphics[width=0.23\textwidth]{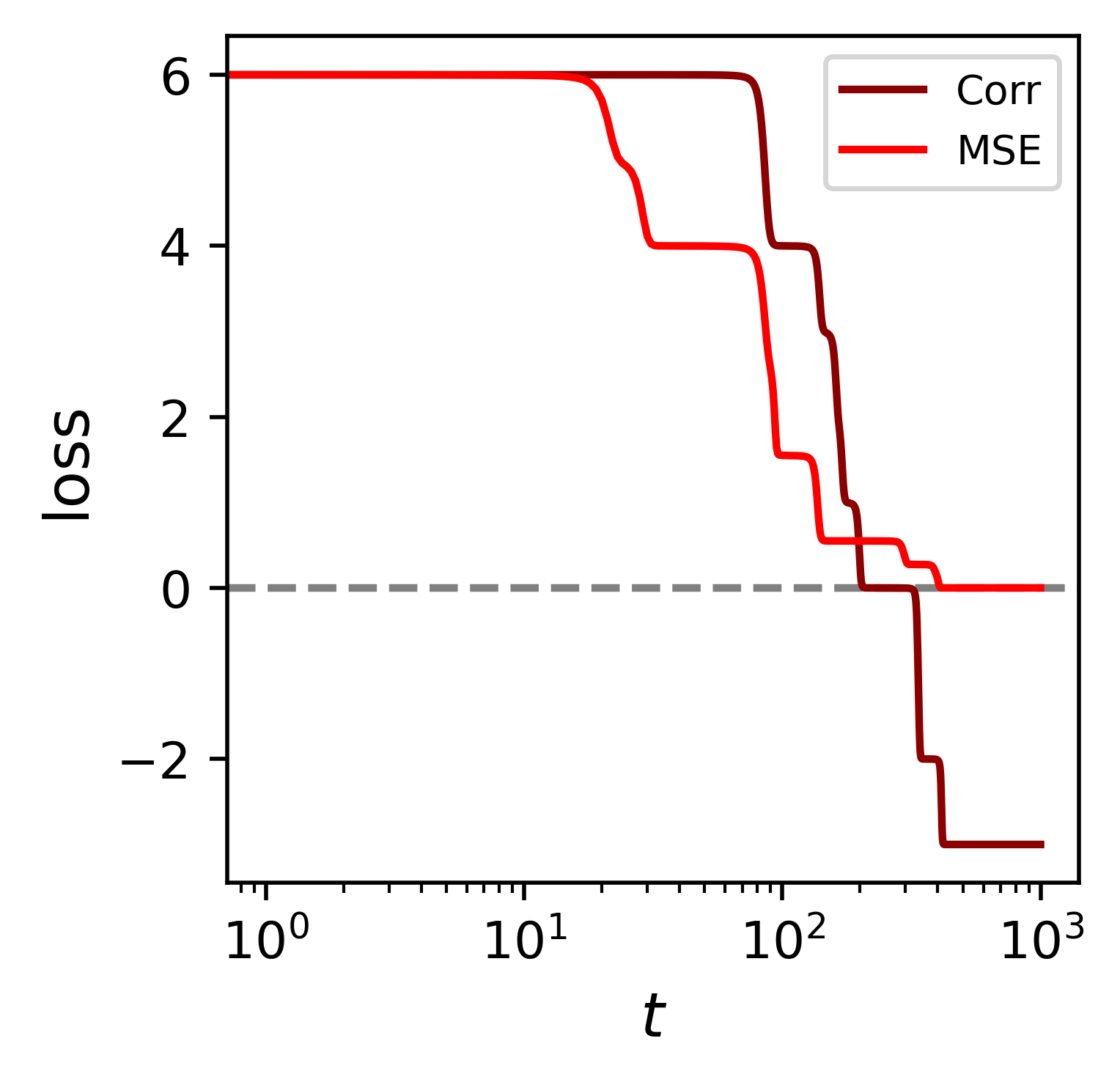}}
    \caption{\textit{Gradient flow trajectories projected to the subspace of orthogonal index vectors for MSE and correlation losses and loss curves; odd activation (a), even activation (b).} 
    Initialization and the number of neurons are fixed; $k \!=\! 2$, $n \!=\! 10$, $d \!=\! 1000$. 
    The winning neurons move toward the closest index vectors also for the MSE loss, but the other neurons move non-trivially due to interactions between them. 
    Adding the repulsion term (MSE loss) virtually decreases the time complexity (bottom row), however, the improvement in the time complexity may be only up to a constant factor.}
    \label{fig:loss-curves}
\end{figure} 

\subsubsection*{Acknowledgements}

BS acknowledges useful discussions with Loucas Pillaud-Vivien and Alberto Bietti in the early phases of this work, and the valuable feedback from Denny Wu and Ludovic Stephan on the manuscript. 
DH acknowledges support from the ONR under grant N00014-24-1-2700.

\newpage

\bibliography{references}

\newpage 

\onecolumn
\appendix 

\section{Reparameterization of the Loss}
\label{sec:reparam-loss-app}

Recall that the correlation loss expands in terms of the neurons $\bw $ and index vectors $\bv_j$ as 
\[
    L(\bw) = C - \sum_{j=1}^k \E [\sigma(\bw^T \bx) \sigma^*(\bv_j^T \bx)].
\]
Due to the rotational symmetry of the standard Gaussian distribution, each term can be expressed as
\[
    \E [\sigma(\bw^T \bx) \sigma^*(\bv_j^T \bx)] = \E_{(x, y) \sim \N(0, C(u))}[\sigma(x) \sigma^*(y)]
\]
where $C(u)$ is a $2 \times 2$ covariance matrix with entries $C(u)_{11} \!=\! 1, C(u)_{22} \!=\! 1, C(u)_{12} \!=\! u, C(u)_{21} \!=\! u$ where $u = \bw^T \bv_j$. 
Hence one can express the above expectation using the variable $u$ only, and the dual activation is well-defined. 
This observation also allows us to express the loss in terms of the dot products $\bw^T \bv_j$ hence 
\begin{align}\label{eq:loss-exp}
  L_0(V^T \bw) = L(\bw).
\end{align}
We will use this expression to derive Lemma~\ref{lem:projection} of the main which is restated below. 

\begin{lemma}\label{lem:projection-app} Assume that $\bw(t)$ solves the ODE ~\eqref{prob:single-neuron}  
given an initial condition $\bw \!=\! \bw_0$. 
Then the vector of dot products $\bu(t) = V^T \bw(t)$ solves the following ODE
\[
     \frac{d}{dt} \bu = - (A - \bu \bu^T) \nabla L_0(\bu) 
\]
with an initial condition $ \bu_0 = V^T \bw_0 $.  
\end{lemma}

\begin{proof} The chain rule allows us to develop the time derivative as follows 
\[
    \frac{d}{dt} V^T \bw = V^T \frac{d}{dt} \bw = - V^T (I_d - \bw \bw^T) \nabla L(\bw). 
\]
The Euclidean gradient of $L$ can be expressed as follows using Eq.~\eqref{eq:loss-exp}
\[
    \nabla L (\bw) = V \nabla L_0(V^T \bw).
\]
Plugging in the Euclidean gradient of $L$, we get 
\begin{align*}
    \frac{d}{dt} V^T \bw &= - V^T (I_d - \bw \bw^T) V \nabla L_0(V^T \bw) \\
    &= - (V^T V - (V^T\bw) (V^T\bw)^T ) \nabla L_0(V^T \bw)
\end{align*}
which completes the derivation. 
Substituting $\bu \!=\! V^T \bw$ and $ A \!=\! V^T V$ gives the expression stated in the Lemma. 
\end{proof}

\paragraph{Derivation of the Boundary}

In order to explicitly describe the domain of $L_0$, we need to understand how the domain of $L(\bw)$, i.e., the unit sphere, changes under the linear projection $V^T: \R^d \to \R^k$. 
This is achieved below, where we characterize the image of the linear projection $V^T: \S^d \to D$ when restricted to the unit sphere.

Let us express $\bw$ as a linear combination of $\bv_j$'s and an orthogonal component to the span of $\bv_j$'s
\[
    \bw = \sum_{j=1}^k \ba_{j} \bv_{j} + \bv^\perp. 
\]
Since $\bw$ is on the unit sphere, the $\ba_{j}$'s should satisfy the following constraint 
\begin{align*}
    \| \bw \|^2_2 = \ba^T (V^T V) \ba + \| \bv^\perp \|^2_2 = 1 \ \ \Rightarrow \ \ \ba^T (V^T V) \ba \leq 1. 
\end{align*}

The dot products are given by $\bu_j = \bw^T \bv_j$. 
Plugging in the expansion of $\bw$ in the basis of $\bv_j$, we get 
\begin{align*}
    \bu_j = \sum_{j'=1}^k \ba_{j'} \bv_j^T \bv_{j'} \ \ \Leftrightarrow \ \ \ba = (V^T V)^{-1} \bu.
\end{align*}
Using the constraint on $\ba$ allows us to derive the constraint on $\bu$ using the above identity as follows 
\begin{align*}
    \ba^T (V^T V) \ba \leq 1 \ \ \Leftrightarrow \ \ \bu^T (V^T V)^{-1} (V^T V) (V^T V)^{-1} \bu \leq 1 \ \ \Leftrightarrow \ \ \bu^T (V^T V)^{-1} \bu \leq 1. 
\end{align*}
Hence the space of dot products is an ellipsoidal ball $D = \{ \bu : \bu^T A^{-1} \bu \leq 1 \}$ where the shape is given by $A^{-1}$.

\paragraph{Eigenvalue and Eigenvector analysis}
Let us denote with $\textbf{1}_k$ and $\textbf{I}_k$ the all-ones vector and the identity matrix of size $k$. We want to determine the eigenvalues and eigenvectors of the matrix 
\[
A = \beta  \textbf{1}_k \textbf{1}_k^T + (1-\beta)\textbf{I}_k. 
\]
First, note that $A$ and $\textbf{1}_k \textbf{1}_k^T$ have the same eigenvectors. If $\textbf{v}$ is an eigenvector of $\textbf{1}_k \textbf{1}_k^T$ with eigenvalue $\lambda$, then 
\begin{align*}
A\textbf{v} & = \beta  \textbf{1}_k \textbf{1}_k^T \textbf{v} + (1-\beta)\textbf{I}_k\textbf{v} \\ 
& = \beta \lambda \textbf{v} + (1-\beta)\textbf{v} = (\beta \lambda + (1-\beta)) \textbf{v}, 
\end{align*}
and $\textbf{v}$ is an eigenvector of $A$ with eigenvalue $(\beta \lambda + (1-\beta))$. Since $\textbf{1}_k \textbf{1}_k^T$ has rank one, it has exactly one non-zero eigenvalue given by $\|{\textbf{1}_k}\|^2 = k$ with eigenspace $\text{span}(\textbf{1}_k)$ and $k-1$ zero eigenvalues with eigenspace $\text{span}(\textbf{1}_k)^{\perp} = \text{span}\{\textbf{e}_1 - \textbf{e}_2, \textbf{e}_1 - \textbf{e}_3, \ldots,  \textbf{e}_1 - \textbf{e}_k\}$, where $\textbf{e}_i$ denotes the $i$th standard unit vector. This implies that $A$ has eigenvalues $1+\beta(k-1)$ and $(1-\beta)$ with eigenspaces $\text{span}(\textbf{1}_k)$ and $\text{span}\{\textbf{e}_1 - \textbf{e}_2, \textbf{e}_1 - \textbf{e}_3, \ldots,  \textbf{e}_1 - \textbf{e}_k\}$ respectively. In particular, this shows that $A$ is invertible for $0 \leq \beta < 1$. 
	
For subsequent calculations, we need to relate the eigenvalues  and eigenvectors of $A=V^TV$ to the eigenvalues and eigenvectors of $VV^T$. This is achieved in the following Lemma. 
\begin{lemma}\label{lem:eigenvalue-equality}
The matrices $V^TV$ and $VV^T$ satisfy the following properties.
\begin{enumerate}[label=(\roman*)]
\item $V^TV$ and $VV^T$ have the same nonzero eigenvalues.
\item Every eigenvector $\bv$ of $VV^T$ with nonzero eigenvalue $\lambda_{\bv}$ satisfies $\bv = V \bw$, where $\bw$ is an eigenvector of $V^TV$ with the same eigenvalue.
\end{enumerate}
\end{lemma}
	
\begin{proof}
Both assertions follow from the following observation. Let $\bw$ be an eigenvector of $V^TV$ with eigenvalue $\lambda_{\bw} \neq 0$, then 
\begin{align*}
V^TV \bw = \lambda_{\bw} \bw \implies V \Bigl( V^TV \Bigr) \bw = \lambda_{\bw} V \bw \implies \Bigl( VV^T \Bigr) V \bw = \lambda_{\bw} V \bw, 
\end{align*}
and $V\bw \neq 0$, since $\lambda_{\bw} \neq 0$. In other words, $V\bw$ is an eigenvector of $VV^T$ with the same nonzero eigenvalue $\lambda_{\bw}$. Repeating the argument with $VV^T$ instead of $V^TV$ also shows that $V^T \bw$ is an eigenvector of $V^TV$ with eigenvalue $\lambda_{\bw}$, whenever $\bw$ is an eigenvector of $VV^T$ with $\lambda_{\bw} \neq 0$.
\end{proof}

\section{Learning Index Vectors}

The low dimensional dynamical system is expressed in terms of the Euclidean gradient of the loss in Lemma~\ref{lem:projection-app}. 
We will use the concrete version where each coordinate of $\bu$ is made explicit in some proofs in the following Subsections. 
Let us make the expression for the Euclidean gradient of $L_0$ explicit using Eq.~\eqref{eq:corr-loss}
\[
    \nabla L_0(\bu) = - \sum_{p \geq p^*} c_p p \ {\bu}^{\odot p-1} \quad \text{where} \quad c_p = a_p b_p.
\]
The time derivative of $\bu$ is expressed as a matrix vector product in Lemma~\ref{lem:projection-app}. 
Carrying out the matrix vector product explicitly, one gets the following expression for the ODE in terms of the Hermite coefficients
\begin{align}\label{eq:ODE}
    \frac{d}{dt} \bu_j = \sum_{p \geq p^*} c_p p \bigl(\bu_j^{p-1} + \sum_{j' \neq j} A_{j j'} \bu_{j'}^{p-1} - \bu_j \sum_{j'} \bu_{j'}^{p}\bigr). \quad \quad \color{blue}{\text{(general index vectors)}}
\end{align}
Note that $a_p b_p \geq 0$ by Assumption~\ref{ass:hermite-coeff} and $A_{j j'} = \bv_j^T \bv_{j'} \geq 0$ by Assumption~\ref{ass:index-vectors}. 
We analyze the dynamical system in the positive quadrant, i.e., $\bu_j \geq 0$ for all $j \in [k]$.
Hence, we can give the following interpretation for the low dimensional ODE~\eqref{eq:ODE} in terms of the three terms: (i) a self-reinforcing term promoting the self-growth of $\bu_j$, (ii) the reinforcing term promoting the growth of $\bu_j$ given the other dot products, and (iii) a damping term decreasing the growth proportional to $\bu_j$ and the $\ell_p$-norm of the vector of dot products. 

\paragraph{Trajectories remain in the positive quadrant}
For the above interpretation to be valid, we need to show that the trajectories are confined to the positive quadrant. 
More formally, that is, the set $S = [0, 1]^k \cap D$ is invariant.
To show that trajectories do not leave the set $S$, it suffices to analyze the boundaries. 
When $ \bu_j=0$ and $\bu_{j'}\geq 0$ for $j' \geq j$, we have the following
\begin{align}\label{eq:dot-prods-nonnneg}
     \frac{d}{dt} \bu_j = \sum_{p \geq p^*} c_p p \sum_{j' \neq j} A_{j j'} \bu_{j'}^{p-1} \geq 0,
\end{align}
due to Assumption~\ref{ass:hermite-coeff} and Assumption~\ref{ass:index-vectors}. 
Hence, $\bu_j$ does not decrease and become negative; in other words, the trajectories do not leave the set $S$. 

\paragraph{The assumption on the positivity of the dot products at initialization}
In most of our analysis, we assume that the dot products are positive at initialization. 
With random initialization, half of the dot products would be positive on average. 
However, negative dot products at initialization require a more refined analysis than we present here (except for the even activation case, where the proofs can be extended to negative dot products using a sign-flipping argument). 
A reasonable guess is that as long as one of the dot products at initialization is positive, the single neuron dynamics exhibits the same time complexity and converges to the nearest index vector.
However, the proof needs refinements for the relaxed assumption on the initialization to be sufficient (i.e., when only one dot product is positive) which is beyond the scope of this work. 

\subsection{Time complexity}
\label{sec:time-complexity-app}

\begin{theorem}[Time complexity]\label{thm:time-complexity-app} Assume that $d$ is large, $\bu_j(0) \!=\! \Theta(d^{-1/2})$ and $\bu_j(0) \!>\! 0$ for all $j \in [k]$. For $p^* \geq 2$, we show the following time complexities
\begin{itemize}
    \item if $p^* \geq 3$, $T=\Theta(d^{p^*/2-1})$ is necessary and sufficient to reach $\| \bu(T) \|_2 = \Theta(1)$, 
    \item if $p^* = 2$, $T=\Theta(\log(d))$ is necessary and sufficient to reach $\| \bu(T) \|_2 = \Theta(1)$.
\end{itemize}
\end{theorem}

\begin{proof}
Since the dot products are assumed to be positive at initialization, they remain non-negative at all times (see Eq.~\ref{eq:dot-prods-nonnneg} and the comments therein). That is, we have $\bu_j(t) \geq 0$ for all $j \in [k]$ and for all $t \geq 0$.

Let us define $s_2 = \bu^{T}A^{-1}\bu$, the square of the norm induced by $A^{-1}$.
The time derivative of $s_2$ is given by
\begin{align}
\frac{1}{2}\frac{d}{dt}s_{2} & = \Big(\sum_{p \geq p^{*}}c_{p}p (A-\bu\bu^{T})\bu^{\odot p-1}  \Big)^{T}A^{-1}\bu  \label{eq:matrix-derivative}\\
& = \sum_{p \geq p^{*}} c_{p}p \Big( (\bu^{\odot p-1})^{T}\Big(A^{T}-(\bu\bu^{T})^{T}\Big)A^{-1}\bu\Big) \nonumber \\
& = \sum_{p \geq p^{*}} c_{p}p \Big( (\bu^{\odot p-1})^{T}\Big(A-\bu \bu ^{T}\Big)A^{-1}\bu \Big) \nonumber \\
& = \sum_{p \geq p^{*}} c_{p}p \Big( (\bu^{\odot p-1})^{T}\bu (1-\bu^{T}A^{-1}\bu)\Big) \nonumber \\ &= \sum_{p \geq p^{*}} c_{p}p \Big( (\bu^{\odot p-1})^{T}\bu (1-s_{2})\Big) \label{eq:s2-sum}, 
\end{align}
where in equation~\eqref{eq:matrix-derivative} we used that
\[
\frac{1}{2} \frac{d}{dt} (\bu^{T}A^{-1}\bu) =  (\frac{d}{dt}\bu)^{T}A^{-1}\bu. 
\]
Since the correlations are positive at all times, we have $\| \bu \|_{p}^p = (\bu^{\odot p-1})^T\bu$ at all times. Furthermore, since we initialize $\bu$ inside the ellipsoidal ball $D = \{ \bu : \bu^T A^{-1} \bu \leq 1 \}$ (which comes from the fact that $\bu = V^T \bw$ for a unit vector $\bw$), we have that $1-s_2 \geq 0$ holds at all times. This implies that $s_2$ is increasing until we hit the boundary of $D$. We can control the $\ell_p$ norm by using an inequality between norms 
\[
r \leq p \quad \Rightarrow \quad \| \bu \|_p \leq \| \bu \|_r \leq k^{\frac{1}{r}-\frac{1}{p}} \| \bu \|_p,
\]
which is a consequence of Hölder's inequality. For $r=2$, this is equivalent to 
\begin{equation}
\label{eq:hoelder-l2}
k^{1-\frac{p}{2}} \| \bu \|_2^p \leq \| \bu \|_p^p  \leq \| \bu \|_2^p .
\end{equation}
Using the following estimate 
\begin{equation}\label{eq:l2-matrix-norm}
\lambda_{\min}(A) \bu^{T}A^{-1}\bu \leq \|\bu\|_{2}^{2} \leq \lambda_{\max}(A) \bu^{T}A^{-1}\bu, 
\end{equation}
and Hölder's inequality~\eqref{eq:hoelder-l2}, we can bound the $\ell_p$ norm in terms of $s_2$
\begin{align*}
k^{1-\frac{p}{2}} \lambda_{\min}(A)^{\frac{p}{2}} ( \bu^{T}A^{-1}\bu)^{\frac{p}{2}} \leq \|\bu\|_{p}^{p} \leq \lambda_{\max}(A)^{\frac{p}{2}} (\bu^{T}A^{-1}\bu)^{\frac{p}{2}} .
\end{align*}
Let us denote $\lambda_{\max}(A)$ and $\lambda_{\min}(A)$ with $\lambda_1$ and $\lambda_k$ respectively. 
For $p^* \geq 2$, the time derivative of $s_2$ can be bounded as 
\begin{align*}
	&\sum_{p \geq p^*} 2 c_p p \lambda_{k}^{\frac{p}{2}} k^{1-\frac{p}{2}}  s_2^\frac{p}{2} (1 - s_2) \leq \frac{d}{dt} s_2 \leq \sum_{p \geq p^*} 2 c_p p \lambda_{1}^{\frac{p}{2}} s_2^\frac{p}{2} (1 - s_2) \\
	&\sum_{p \geq p^*} 2 C_p \lambda_{k}^{\frac{p}{2}}s_2^\frac{p}{2} (1 - s_2) \leq \frac{d}{dt} s_2 \leq \sum_{p \geq p^*} 2 \bar{C}_p \lambda_{1}^{\frac{p}{2}} s_2^\frac{p}{2} (1 - s_2) \quad \text{where} \quad C_p = c_p p k^{1-\frac{p}{2}}, \bar{C}_p = c_p p.
\end{align*}
Since the $c_p$'s are all positive, we can drop the higher order terms and get the following lower bound
\[
2 C_{p^*}\lambda_{k}^{\frac{p^*}{2}} s_2^\frac{p^*}{2} (1 - s_2) \leq \frac{d}{dt} s_2 \leq 2 \bigl(\sum_{p \geq p^*} \bar{C}_p \lambda_{1}^{\frac{p}{2}}\bigr) s_2^\frac{p^*}{2} (1 - s_2),
\]
where the upper bound comes from observing that $s_2$ does not exceed one ($s_2 \leq 1$) for all times due to the problem geometry. 
To simplify, let us study the dynamics for the interval $s_2 \in [b_0, 1/2]$. 
Up to a change in the constant of the lower bound, we get the following lower and upper bounds 
\[
C_{p^*} s_2^\frac{p^*}{2} \leq \frac{d}{dt} s_2 \leq 2 \bigl(\sum_{p \geq p^*} \bar{C}_p \lambda_{1}^{\frac{p}{2}}\bigr) s_2^\frac{p^*}{2}. 
\]
Using separation of variables, we can integrate the sandwiched part 
\[
\int_{b_0}^{b_1} s_2^\frac{-p^*}{2} d s_2 
\]
which is lower and upper bounded by $C t$ and $\bar{C} t$ for some constants $C$ and $\bar{C}$.
Let us do the integration first for $p^* \geq 3$
\[
\frac{s_2^{\frac{-p^*}{2}+1}}{\frac{-p^*}{2}+1} \biggl|_{b_0}^{b_1} = \frac{1}{\frac{p^*}{2}-1} \bigl(-b_1^{\frac{-p^*}{2}+1} + b_0^{\frac{-p^*}{2}+1} \bigr) = \Theta(d^{\frac{p^*}{2}-1}),
\]
since $b_1 = \Theta(1)$ and $b_0 = \Theta(d^{-1})$ (because every correlation is initialized with $\Theta(d^{-1/2})$).
It remains to do the integration for $p^*=2$ which is simply
\[
\int_{b_0}^{b_1} s_2^{-1} d s_2 = \log(s_2) \biggl|_{b_0}^{b_1} = \log(b_1) - \log(b_0) = \Theta(\log(d)). 
\] 
Finally, the estimate in~\eqref{eq:l2-matrix-norm} implies that the time needed for the $\ell_2$ norm of $\bu$ to reach a non-vanishing value is of the same order as the time needed for $s_2$.
\end{proof}

\subsection{Directional Convergence}
\label{sec:directional-convergence-app}

This subsection applies to orthogonal index vectors. 
In this case $A_{j j'} = 0$ for $j \neq j'$, hence the ODE~\eqref{eq:ODE} simplifies
\begin{align}\label{eq:ODE-orth}
    \frac{d}{dt} \bu_j = \sum_{p \geq p^*} c_p p \bigl(\bu_j^{p-1} - \bu_j \sum_{j'} \bu_{j'}^{p}\bigr). \quad \quad \color{blue}{\text{(orthogonal index vectors)}}
\end{align}

First, we make a helper lemma for the invariance. 

\begin{lemma}[Symmetry breaking happens at initialization]\label{lem:sym-breaking-app} 
Wlog assume that $ |\bu_1(0)| = \max_{j \in \{1,...,k\} } |\bu_j(0)| $. 
Assume that $\sigma^*$ has an information exponent $p^* \! \geq \! 2$. 
We show the following 
\begin{enumerate}[label=(\roman*)]
    \item if all correlations are positive at initialization, then the maximum correlation is preserved, i.e. $\bu_1(t) = \max_{j \in \{1,...,k\} } \bu_j(t)$ for any $\sigma^*$,
    \item if $\sigma^* $ is even, then $ |\bu_1(t)| = \max_{j \in \{1,...,k\} } |\bu_j(t)|$ for any initialization.
\end{enumerate}
\end{lemma}

Lemma~\ref{lem:sym-breaking-app} gives us that if all correlations are positive at initialization, $j^* = \text{argmax}_{j \in \{1, ..., k \}} \bw_i(t)^T \bv_j$ is fixed over time. 
The same argument holds up to a sign change for even activation functions and arbitrary initializations. 

\begin{proof} Let us start with the simplest case $(i)$ where the dot products are assumed to be positive at initialization.
Since the dot products are assumed to be positive at initialization, they remain non-negative at all times (see Eq.~\ref{eq:dot-prods-nonnneg} and the comments therein). That is, we have $\bu_j(t) \geq 0$ for all $j \in [k]$ and for all $t \geq 0$.

Since $\bu_1$ is assumed to be the leading correlation at initialization, this implies that
\[
    \bu_1(0) = \max_{j \in \{1,...,k\} } \bu_j(0) > 0.
\]
We want to show 
\[
    \bu_1(t) = \max_{j \in \{1,...,k\} } \bu_j(t) \quad \text{for all times} \ t.
\]
Our strategy is to show that $\frac{d}{dt}(\bu_1- \bu_j)$ is non-negative if $\bu_1-\bu_j$ is positive. 
We will use that $\bu_1 - \bu_{j'}$ is non-negative for all $j'$. 
Taking the difference between the time derivative of $\bu_1$ and $\bu_j$ gives
\begin{align}\label{eq:ODE-diff}
    \frac{d}{dt} (\bu_1 - \bu_j) = \sum_{p \geq p^*} c_p p \bigl( (\bu_1^{p-1} - \bu_j^{p-1}) - (\bu_1 -\bu_j) \sum_{j'=1}^k \bu_{j'}^{p} \bigr).
\end{align}
The following chain of inequalities completes the argument
\begin{align*}
    \bu_1^{p-1} - \bu_j^{p-1} &\geq (\bu_1^{p-1} - \bu_j^{p-1})(\sum_{j'=1}^k \bu_{j'}^2) \hspace{5.0 cm} \text{\color{gray} since $\sum \bu_j^2 \leq 1$ and $\bu_1 \geq \bu_j$} \\
    &= (\bu_1 - \bu_j)(\bu_1^{p-2} + \bu_1^{p-3} \bu_j + ... + \bu_j^{p-2}) (\sum_{j'=1}^k \bu_{j'}^2) \\
    &\geq (\bu_1 - \bu_j) \bu_1^{p-2} (\sum_{j'=1}^k \bu_{j'}^2) \hspace{7.5 cm} \text{\color{gray} since $\bu_j \geq 0$} \\
    &\geq (\bu_1 - \bu_j) (\sum_{j'=1}^k \bu_{j'}^p) \hspace{6.5 cm} \text{\color{gray} since $\bu_1 \geq \bu_{j'}$ for all $j'$}
\end{align*}

There is a remaining step to complete the proof of the case $(i)$: the leading correlation might not be unique at initialization. 
That is the case $\bu_1(0)=...=\bu_{\ell}(0) > \max_{j \in \{ \ell+1, ..., k \} } \bu_j(0)$. 
Observe that Eq.~\ref{eq:ODE-diff} ensures that $\bu_1(t)=...=\bu_{\ell}(t)$ for all times and the argument above applies to the differences $\bu_1 - \bu_j$ for $j \in \{ \ell+1, ..., k \}$. This completes the proof of case $(i)$. 

We will handle the case $(ii)$ very similarly by using a sign flip argument.
First, let us show that the odd Hermite coefficients of $\sigma^*$ vanish to zero when $\sigma^*$ is even. 

Using Rodrigues' formula, we get that $h_p$ is an odd function when $p$ is odd (see the \href{https://en.wikipedia.org/wiki/Hermite_polynomials}{Wikipedia page}). 
Hence the integrand below is odd and integrating it from $-\infty$ to $\infty$ gives a zero, i.e., 
\begin{align*}
    \frac{1}{\sqrt{2 \pi}} \int_{-\infty}^\infty \sigma^*(x) h_p(x) \exp(-\frac{x^2}{2}) dx = 0.
\end{align*}
 
Therefore, in this case, for the ODE in Eq.~\ref{eq:ODE-orth}, the terms corresponding to the Hermite modes $p$ is zero when $p$ is odd. 
Let us compute the time derivative of $-\bu_j$
\begin{align*}
    \frac{d}{dt} (-\bu_j) &= (- \bu_j) \sum_{p \ \text{even}} c_p p (\bu_j^{p-2} - \sum_{j'=1}^k \bu_{j'}^p ), \\
    &= (- \bu_j) \sum_{p \ \text{even}} c_p p ( (-\bu_j)^{p-2} - \sum_{j'=1}^k (-\bu_{j'})^p ).
\end{align*}
The second equality holds since all Hermite modes are even. 
Hence, if $(\bu_1(t), ..., \bu_k(t))$ is a solution of this ODE with an initial condition $(\bu_1(0), ..., \bu_k(0))$, then $ (\xi_1 \bu_1(t), ..., \xi_k \bu_k(t))$ is also a solution with the initial condition $(\xi_1 \bu_1(0), ..., \xi_k \bu_k(0))$ where $\xi_i \in \{ \pm 1 \}$. 

Note that for any initialization, we can flip the sign of the negative correlations. 
Applying the result of case $(i)$, we get that the maximum correlation $\xi_1 \bu_1(t)$ is preserved for all times. This completes the proof of case $(ii)$. 
\end{proof}

\begin{proposition}[Directional Convergence]\label{prop:dir-converge-app} 
Let $\ell \!\in\! [k]$ and $I \!=\! [k] \setminus [\ell] $. 
Assume that $\bu_j(0) \!>\! 0$ for all $j \!\in\! [k]$ and wlog $\bu_1(0)\!=\! \dotsb \!=\! \bu_{\ell}(0) \!>\! \max_{j \in I} \bu_j(0) $. 
For any $\sigma^*$ with information exponent $p^* \!\geq \! 3$, the dynamics converge to 
\[
    \lim_{t \to \infty} \bu_j(t) \!=\! \frac{1}{\sqrt{\ell}} \ \ \text{for} \ \ j \in [\ell].
\]
Moreover, if $\sigma^*$ is even and with information exponent $p^* \!\geq\! 4$, assume that wlog $|\bu_1(0)| \!=\!\dotsb\!=\! |\bu_{\ell}(0)| \!>\! \max_{j \in I} |\bu_j(0)|$, the dynamics converge to 
\[
    \lim_{t \to \infty} \bu_j(t) \!=\! \frac{\operatorname{sgn}(\bu_j(0))}{\sqrt{\ell}} \ \ \text{for} \ \ j \in [\ell].  
\]
This implies $\lim_{t \to \infty} \bu_j(t) \!=\! 0 \ \ \text{for} \ j \!\in\! I$ for both cases. 
\end{proposition}

\begin{proof} 
Since the dot products are assumed to be positive at initialization, they remain non-negative at all times (see Eq.~\ref{eq:dot-prods-nonnneg} and the comments therein). That is, we have $\bu_j(t) \geq 0$ for all $j \in [k]$ and for all $t \geq 0$.

Let us recall the evolution of the total alignment from the proof of Theorem~\ref{thm:time-complexity-app}
\[
    \frac{1}{2} \frac{d}{dt} s_2 = \sum_{p \geq p^*} c_p p (\sum_{j=1}^k \bu_j^p) (1 - s_2).
\]
For all $s_2 \in (0,1)$, the time derivative is positive and $s_2$ increases over time until it reaches $s_2 = 1$. 
Once $s_2 = 1$ is satisfied, the time derivative vanishes and $s_2$ remains constant.
In other words, once $\bw$ reaches the subspace of the index vectors, it remains in that subspace for all times. 

To specify in which direction this convergence happens, we analyze the fixed points of the ODE. 
The fixed points correspond to the solution of the following non-linear system of equations with $k$ variables
\[ 
    \sum_{p \geq p^*} c_p p \bigl( \bu_j^{p-1} - \bu_j \sum_{j'=1}^k \bu_{j'}^{p}\bigr) = 0 \quad \text{for} \quad j\! \in \! \{1,...,k\}.
\]
From the proof of Lemma~\ref{lem:sym-breaking-app}, the following holds for all times
\begin{align}\label{eq:ordering}
    \bu_1(t)=...=\bu_{\ell}(t) \!>\! \max_{j \in \{ \ell\!+\!1,...,k \} } \bu_j(t) \! \geq \! 0.
\end{align}
Moreover, note that for $j \in \{ 1,...,\ell \}$, and for $p \!\geq\! 3$
\begin{align}\label{eq:ineq-fixed-point}
    \bu_j^{p-1} \geq \bu_j^{p-1} (\sum_{j=1}^k \bu_j^2) \geq \bu_j \sum_{j'=1}^k \bu_{j'}^p 
\end{align}
where the first inequality is due to $ \sum \bu_j^2 \!\leq\! 1$ and the second one is due to $\bu_j \!>\! \bu_{j'}$ for $j' \in \{ \ell\!+\!1, ..., k \}$, $\bu_j \!>\! 0$, and $\bu_{j'} \!\geq\! 0$ for all times (see Eq.~\ref{eq:ordering}). 

For configurations for which the above inequality (Eq.~\ref{eq:ineq-fixed-point}) is strict, we have 
\[
    \sum_{p \geq p^*} c_p p \bigl(\bu_j^{p-1} - \bu_j \sum_{j'} \bu_{j'}^{p}\bigr) > 0.
\]
For configurations that are fixed points, the above inequality (Eq.~\ref{eq:ineq-fixed-point}) must be tight which implies the following 
\[
    \sum_{j=1}^k \bu_j^2 = 1 \quad \text{and} \quad \bu_{j'} = 0 \quad \text{for all} \ \ j' \in \{ \ell\!+\!1, ..., k\}.
\]
Finally, setting $\bu_j = \alpha$ for $j \in \{ 1,..., \ell \}$ and $\bu_{j'}=0$ for $j' \in \{ \ell\!+\!1,...,k \}$ gives $\alpha \!=\! \frac{1}{\sqrt{\ell}}$ which completes the proof for the first case.

We can handle the second case similarly by using a sign flip argument as done in the proof of Lemma~\ref{lem:sym-breaking-app}. 
In particular, consider the dynamics of $(\xi_1 \bu_1, ..., \xi_k \bu_k)$ with an initial condition $(\xi_1 \bu_1(0), ..., \xi_k \bu_k(0))$ where $\xi_i = \text{sgn}(\bu_i(0))$. 
The first case then shows that the dynamics converge to 
\[
    \lim_{t \to \infty} \xi_i \bu_i(t) = \frac{1}{\sqrt{\ell}}
\]
which is equivalent to the statement in the second case up to a multiplication of both sides with $\xi_i$. 

\end{proof}

\section{Saddle-to-Minimum Transition}
\label{sec:app-saddle-to-min}

First, we will evaluate the loss function at the `average' fixed point given by 
\begin{align*}
    \bar{\bold{w}} = \frac{\bold{v}_1 + \dots + \bold{v}_k}{\| \bold{v}_1 + \dots + \bold{v}_k \|}, 
\end{align*}
where 
\begin{align*}
    \| \bold{v}_1 + \dots + \bold{v}_k \|^2 = k + k(k-1) \beta , \quad \beta \in [0, 1].
\end{align*}

The dot product between $\bar{\bold{w}}$ and the index vector $\bold{v}_j$ is given by 
\begin{align*}
    \bar{\bw} \cdot \bv_j = \Bigl(\frac{1+(k-1) \beta}{k}\Bigr)^{1/2}.
\end{align*}
Using the expression in equation~\eqref{eq:corr-loss}, the loss function evaluated at  $\bar{\bold{w}}$ can be computed as follows 
\begin{align*}
L_\beta (\bar{\bw})  & = C - \sum_{p \geq p^*} c_p \sum_{j=1}^k (\bar{\bold{w}}^T \bv_j) ^p\\
& = C - \sum_{p \geq p^*} c_p k \Bigl(\frac{1+(k-1) \beta}{k}\Bigr)^{p/2}.
\end{align*}

In particular, evaluating the loss function at $\beta =0$ and $\beta =1$ implies
\begin{align*}
    L_\beta (\bar{\bw}) \bigl|_{\beta=0} = C - \sum_{p \geq p^*} c_p \Bigl(\frac{1}{k}\Bigr)^{p/2} k, \quad L_\beta (\bar{\bw}) \bigl|_{\beta=1} = C - \sum_{p \geq p^*} c_p k.
\end{align*}

Moreover, the derivative with respect to $\beta$ satisfies
\begin{align*}
    \frac{d}{d \beta} L_\beta(\bar{\bw}) = -\sum_{p \geq p^*} \frac{c_p p}{2} \Bigl(\frac{1+(k-1) \beta}{k}\Bigr)^{p/2-1} (k-1) < 0.
\end{align*}

Intuitively, the correlation between $\smash{f(\bx) = \sigma(\bar{\bw}^T \bx)}$ and the multi-index model increases as the target vectors approach each other, i.e., $\beta$ increases, hence decreasing the correlation loss.

The manifold geometry of the unit sphere makes it complicated to compute the Hessian. 
One may need to change the coordinate system to polar coordinates, which makes computations cumbersome for $d \geq 3$. 

We consider all paths on the unit sphere passing through $\bar{\bw}$.
In particular, we consider the unit circle spanned by $\bar{\bw}$ and $\bv$ such that $\bar{\bw}^T \bv = 0$, that is 
\begin{align*}
    \bw(\theta) = \bar{\bw} \cos(\theta) + \bv \sin(\theta).
\end{align*}

The loss on the circle is then one-dimensional
\begin{align*}
    L_\beta(\theta) = C - \sum_{p = p^*}^P c_p \sum_{j=1}^k ((\bar{\bw} \cos(\theta) + \bv \sin(\theta))^T \bv_j)^p
\end{align*}
which can be viewed as a periodic function with period $2 \pi$. 
This allows us to compute derivatives without worrying about the reparameterization of the unit sphere. 
Note that we truncated the series at the $P$-th term due to technical reasons. 
This corresponds to choosing a polynomial target activation function $\sigma^*$ with degree $P$ and information exponent $p^*$.
Let us restate Theorem~\ref{thm:saddle-min} below and then give the proof. 

\begin{theorem}[Saddle-to-Minimum]\label{thm:saddle-min-app} 
Assume $d > k$. $\bar{\bw}$ is a strict saddle when the dot product is upper bounded by 
\[
    \beta < \frac{p^*-2}{k+p^*-2},
\]
whereas $\bar{\bw}$ is a local minimum when the dot product is lower bounded by 
\[
    \frac{P-2}{k+P-2} < \beta. 
\]
Therefore, for the case $\sigma^* \!=\! h_{p^*}$
\[
    \beta_c = \frac{p^*-2}{k+p^*-2}
\]
is the sharp threshold characterizing the transition from a strict saddle to a local minimum. 
\end{theorem}

\begin{proof} 
The first derivative is given by 
\begin{align*}
    \frac{d}{d \theta} L_\beta (\theta) = - \sum_{p = p^*}^P c_p p \sum_{j=1}^k ((\bar{\bw} \cos(\theta) + \bv \sin(\theta))^T \bv_j)^{p-1} (-\bar{\bw} \sin(\theta) + \bv \cos(\theta))^T \bv_j.
\end{align*}
Evaluating the first derivative at zero, we get 
\begin{align*}
    \frac{d}{d \theta} L_\beta (\theta) \bigg|_{\theta=0} = - \sum_{p = p^*}^P c_p p \sum_{j=1}^k ( \bar{\bw}^T \bv_j)^{p-1} (\bv^T \bv_j).
\end{align*}

The inner sum can be viewed as a dot product between $[\alpha, ..., \alpha]$ and $[\bv^T \bv_1, ..., \bv^T \bv_k]$. 
Observe that the evaluation of the derivative at $\theta = 0$ vanishes, since $\sum_{j=1}^k \bv^T \bv_j = c \bv^T \bar{\bw} = 0$, which implies that $\bar{\bw}$ is a fixed point. 
To study the curvature, we need to compute the second derivatives. 

The second derivative is given by 
\begin{align*}
    \frac{d^2}{d \theta^2} L_\beta (\theta) = - \sum_{p = p^*}^P c_p p \sum_{j=1}^k \biggl( &(p-1)(\bar{\bw} \cos(\theta) + \bv \sin(\theta))^T \bv_j)^{p-2} ((-\bar{\bw} \sin(\theta)+ \bv \cos(\theta))^T \bv_j)^2  \\
    &+ ((\bar{\bw} \cos(\theta) + \bv \sin(\theta))^T \bv_j)^{p-1} (-\bar{\bw} \cos(\theta) - \bv \sin(\theta))^T \bv_j \biggr ).
\end{align*}

Evaluating the second derivative at zero, we get 
\begin{align}\label{eq:second-derivative}
    \frac{d^2}{d \theta^2} L_\beta (\theta) \bigg|_{\theta=0} = \sum_{p = p^*}^P c_p p \sum_{j=1}^k \bigl( - (p-1) (\bar{\bw}^T \bv_j)^{p-2} (\bv^T \bv_j)^2 + (\bar{\bw}^T \bv_j)^{p} \bigr ).
\end{align}

The curvature at $\bar{\bw}$ is positive in the direction of $\bv$ if the following holds for all $p \geq p^*$ 
\begin{align*}
    & \sum_{j=1}^k (p-1) (\bar{\bw}^T \bv_j)^{p-2} (\bv^T \bv_j)^2 < \sum_{j=1}^k (\bar{\bw}^T \bv_j)^{p} \quad \Leftrightarrow \\
    & (p-1) (\frac{1+\beta(k-1)}{k})^{p/2-1} \sum_{j=1}^k (\bv^T \bv_j)^2 < k \bigl(\frac{1+\beta(k-1)}{k}\bigr)^{p/2} \quad \Leftrightarrow \\
    & (p-1) \sum_{j=1}^k (\bv^T \bv_j)^2 < (1+\beta(k-1)). 
\end{align*}
If $\bv \perp \bold{V}$ where $\bold{V} = \text{Span} ( \bv_1, ..., \bv_k )$, the above inequality is trivially true. 
We will study the case $ \bv \in \bold{V} $ in what follows by giving an upper bound on the LHS term. 
By Lemma~\eqref{lem:eigenvalue-equality} we know that the leading eigenvector of $V V^T$ is $\bar{\bw}$ and the orthogonal space to this direction in the subspace $\bold{V} = \text{Span} ( \bv_1, ..., \bv_k ) $ has eigenvalues $1-\beta$. 
This implies that
\begin{align*}
    \max_{\substack{\bv^T \bar{\bw}=0 \\ \bv \in \bold{V}}} \ \bigl \| \begin{bmatrix}
        \bv_1^T \\
        \vdots \\
        \bv_k^T 
    \end{bmatrix} \bv \bigr \|_2^2 =  \max_{\substack{\bv^T \bar{\bw}=0 \\ \bv \in \bold{V}}} \ \bv^T \begin{bmatrix}
        \bv_1 & \dots & \bv_k 
    \end{bmatrix} \begin{bmatrix}
        \bv_1^T \\
        \vdots \\
        \bv_k^T 
    \end{bmatrix} \bv = \lambda_{\max, 2} (V V^T) = \lambda_{\max, 2} (V^T V) = 1-\beta
\end{align*}
where $\lambda_{\max, 2}$ denotes the second largest eigenvalue. 
The following condition gives us the desired upper bound on the LHS
\begin{align*}
    (p-1) (1-\beta) < 1 + \beta(k-1) \ \ \Leftrightarrow \ \ \beta > \frac{p-2}{k+p-2}.
\end{align*}

We want this to be satisfied for all $P \geq p \geq p^*$, and the necessary and sufficient condition for this is 
\begin{align}\label{eq:beta-lower-bound}
    \beta > \frac{P-2}{k+P-2}. 
\end{align}
This argument applies to any $\bv \perp \bar{\bw}$. 
Wlog, consider $\bar{\bw}$ as the north pole, and an arbitrary $\bv \perp \bar{\bw}$ located on the equator. 
For every path trough $\bar{\bw}$ we can find a $\bv \perp \bar{\bw}$ on the equator, such that the circle through $\bv$ and $\bar{\bw}$ coincides with the given path in a neighborhood of $\bar{\bw}$.
Hence we conclude that $\bar{\bw}$ is a local minimum if $\beta$ is bigger than the lower bound~\eqref{eq:beta-lower-bound}.

The curvature at $\bar{\bw}$ is negative in the direction of $\bv$ if the following holds for all $p \geq p^*$ 
\begin{align*}
    (1+\beta(k-1)) < (p-1) \sum_{j=1}^k (\bv^T \bv_j)^2.
\end{align*}
This calls for a lower bound on the RHS term. Let us first consider the case when $\bv \in \bold{V}$. 
Such a vector $\bv$ must lie in the orthogonal complement of the kernel of $V^T$, since the image of $V$ is orthogonal to the kernel of $V^T$. In particular, $V^T\bv \neq0$ and hence $\| V^T\bv\|_2^{2} > 0$. 
This allows us to argue that
\begin{align*}
    \min_{\substack{\bv^T \bar{\bw}=0 \\ \bv \in \bold{V}}} \ \bigl \| \begin{bmatrix}
        \bv_1^T \\
        \vdots \\
        \bv_k^T 
    \end{bmatrix} \bv \bigr \|_2^2 =  \min_{\substack{\bv^T \bar{\bw}=0 \\ \bv \in \bold{V}}} \ \bv^T \begin{bmatrix}
        \bv_1 & \dots & \bv_k 
    \end{bmatrix} \begin{bmatrix}
        \bv_1^T \\
        \vdots \\
        \bv_k^T 
    \end{bmatrix} \bv = \lambda_{\min \neq 0} (V V^T) = \lambda_{\min} (V^T V) = 1-\beta, 
\end{align*}
where $\lambda_{\min \neq 0}(VV^T) $ is the smallest nonzero eigenvalue of $VV^T$. 
The following condition gives us the desired lower bound on the RHS.
\begin{align*}
    (1+\beta(k-1)) < (p-1)(1-\beta) \ \ \Leftrightarrow \ \ \beta < \frac{p-2}{k+p-2}
\end{align*}
for all $p \geq p^*$. 
Then the necessary and sufficient condition for this is 
\begin{align}\label{eq:beta-upper-bound}
    \beta < \frac{p^*-2}{k+p^*-2}. 
\end{align}

In particular, if $d=k$, the curvature is negative in every direction since $\mathbb{S}^{d-1} \subset \bold{V}$. 
This implies that $\bar{\bw}$ is a local maximum and not a saddle. 
However, when $k<d$, we can compute the curvature in a direction $\bv^\perp$ that is orthogonal to the span of $\{ \bv_1, ..., \bv_k \}$ and also orthogonal to $\bar{\bw}$ as a result. 
Evaluating Eq.~\eqref{eq:second-derivative}, observe that the first term is zero and the second term is positive, hence yielding a positive second derivative in this direction. 

Therefore, under condition~\eqref{eq:beta-upper-bound}, $\bar{\bw}$ is a saddle point and it has index $k$. 

\end{proof}

One may wonder whether $\bar{\bw}$ is a fixed point for more general dot product matrices. 
However, this situation is specific to the equiangular frame. For example, consider $k=3$ and the following dot product matrix
\[
    \begin{bmatrix}
    1 & \beta_1 & \beta_2 \\
    \beta_1 & 1 & \beta_2 \\
    \beta_1 & \beta_2 & 1
    \end{bmatrix} \quad \quad \text{where} \quad \beta_1 \neq \beta_2.
\]
The average point $\bar{\bw}$ is not a fixed point in this case as $ \nabla L (\bar{\bw})$ is not parallel to $ \bar{\bw} $. 

\newpage

\section{Supplementary Simulations}

\subsection{Bifurcation as $k$ increases}

\begin{figure*}[h]
     \centering
     {\includegraphics[width=0.4\textwidth]{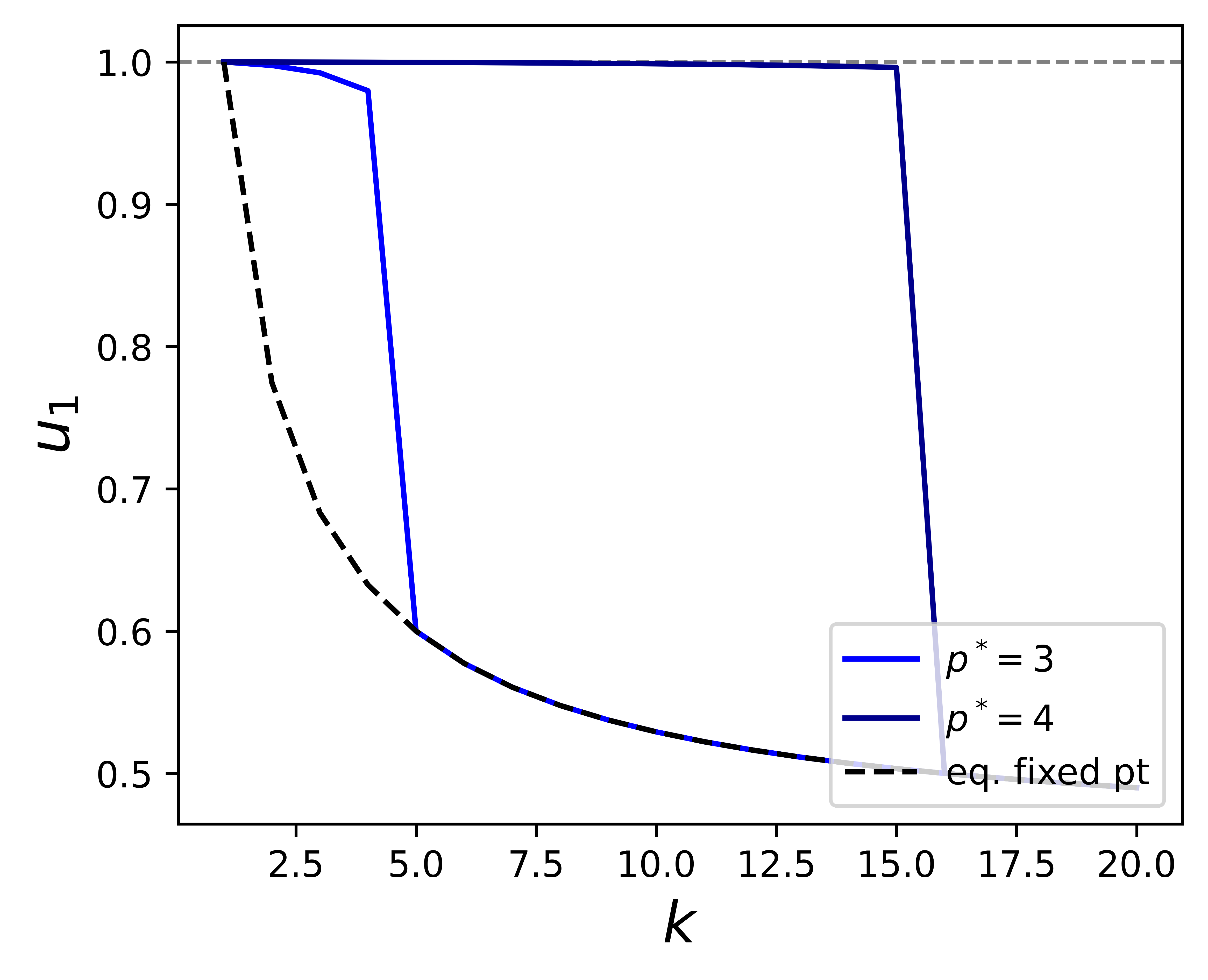}}
     {\includegraphics[width=0.4\textwidth]{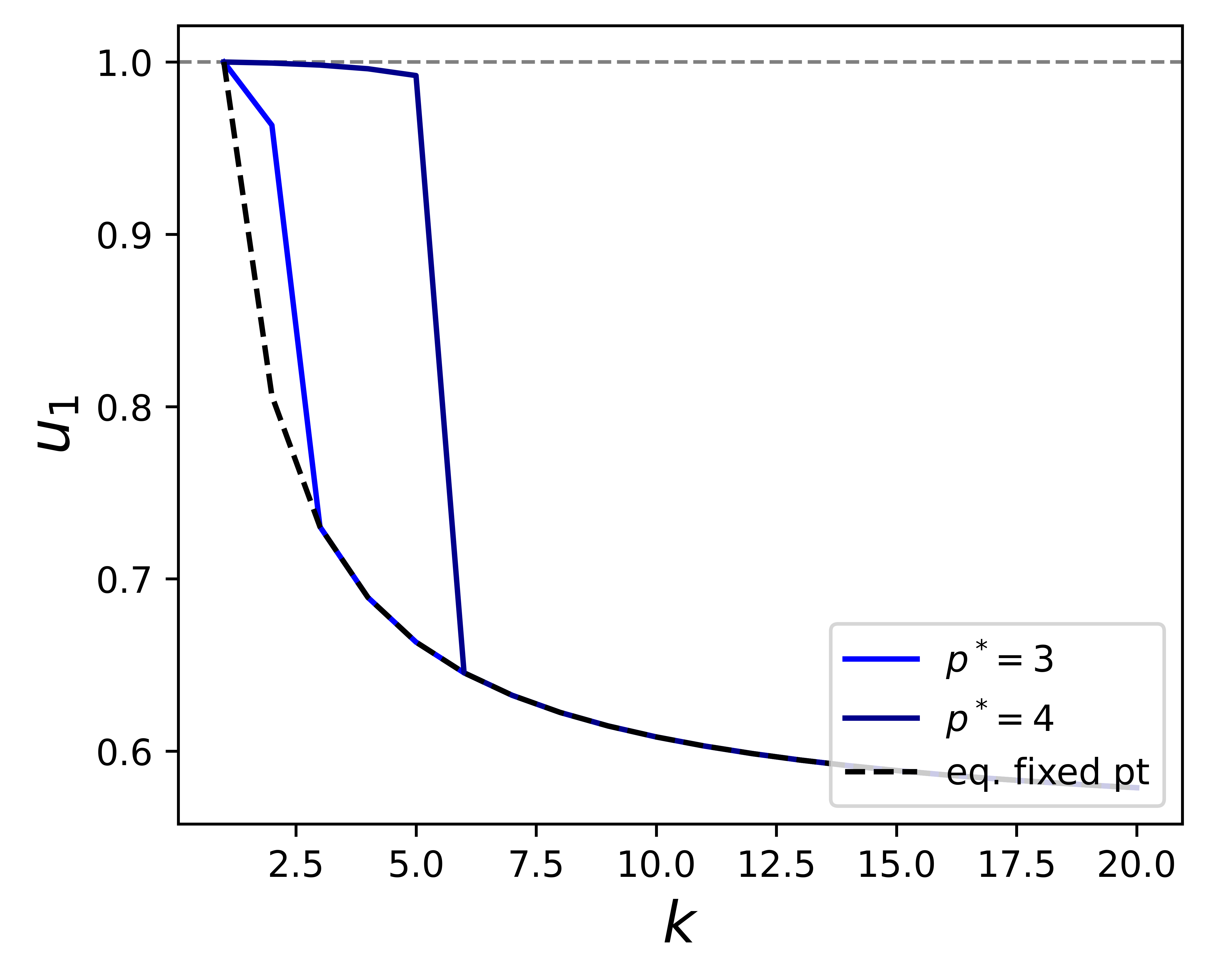}} \\
     \vspace{0.3cm}
    \caption{\textit{Maximum dot product at convergence as the number of index vectors $k$ increases from $1$ to $20$.} (left) $\beta=0.2$ and (right) $\beta=0.3$. Observe that increasing the number of index vectors pulls the flow away from preferring one of the index vectors to the average of the index vectors as indicated by the black dashed line. 
    }
\end{figure*}

\subsection{Phase Portraits}

\begin{figure*}[h]
     \centering
     {\includegraphics[width=0.3\textwidth]{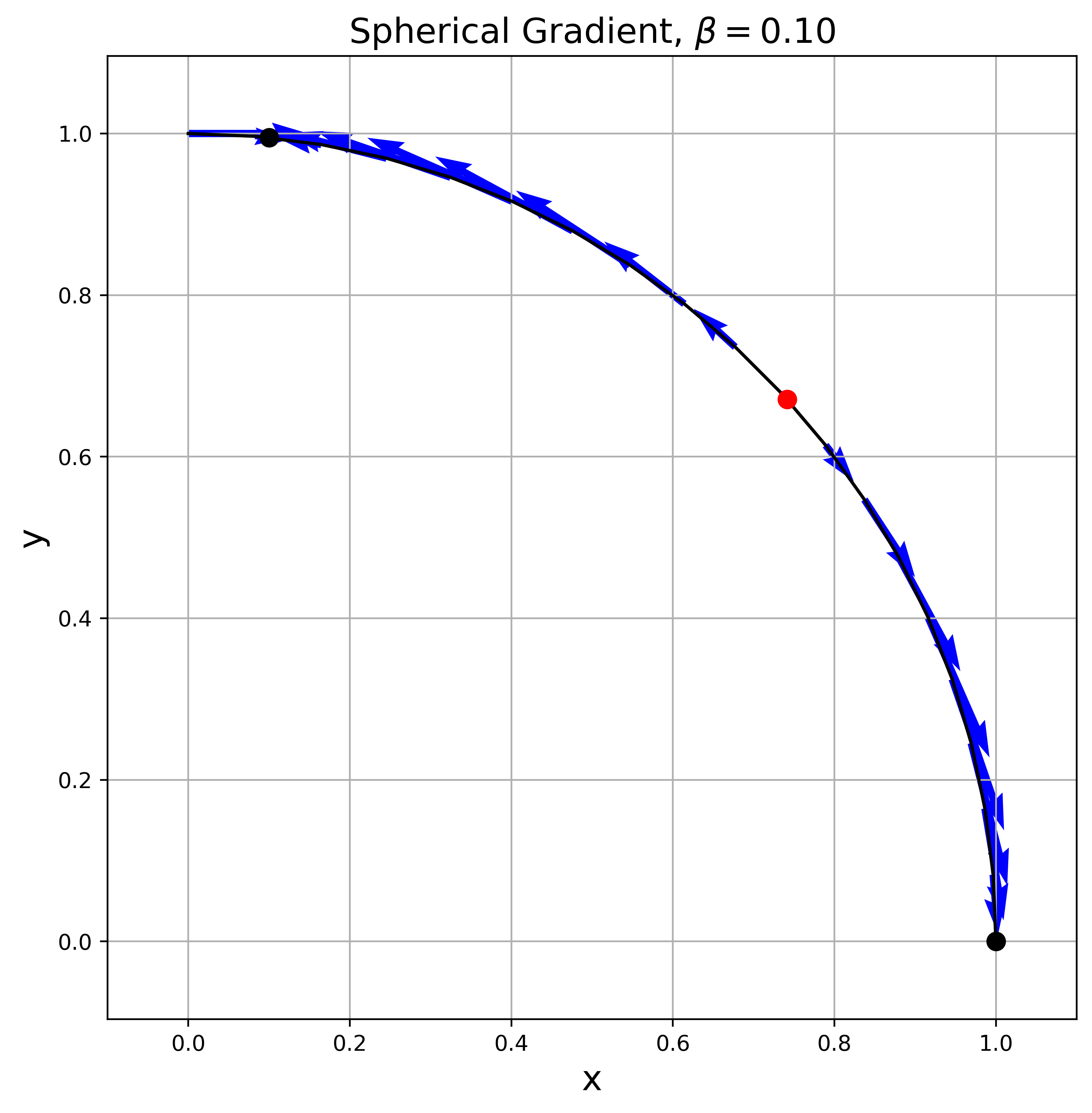}}
     {\includegraphics[width=0.3\textwidth]{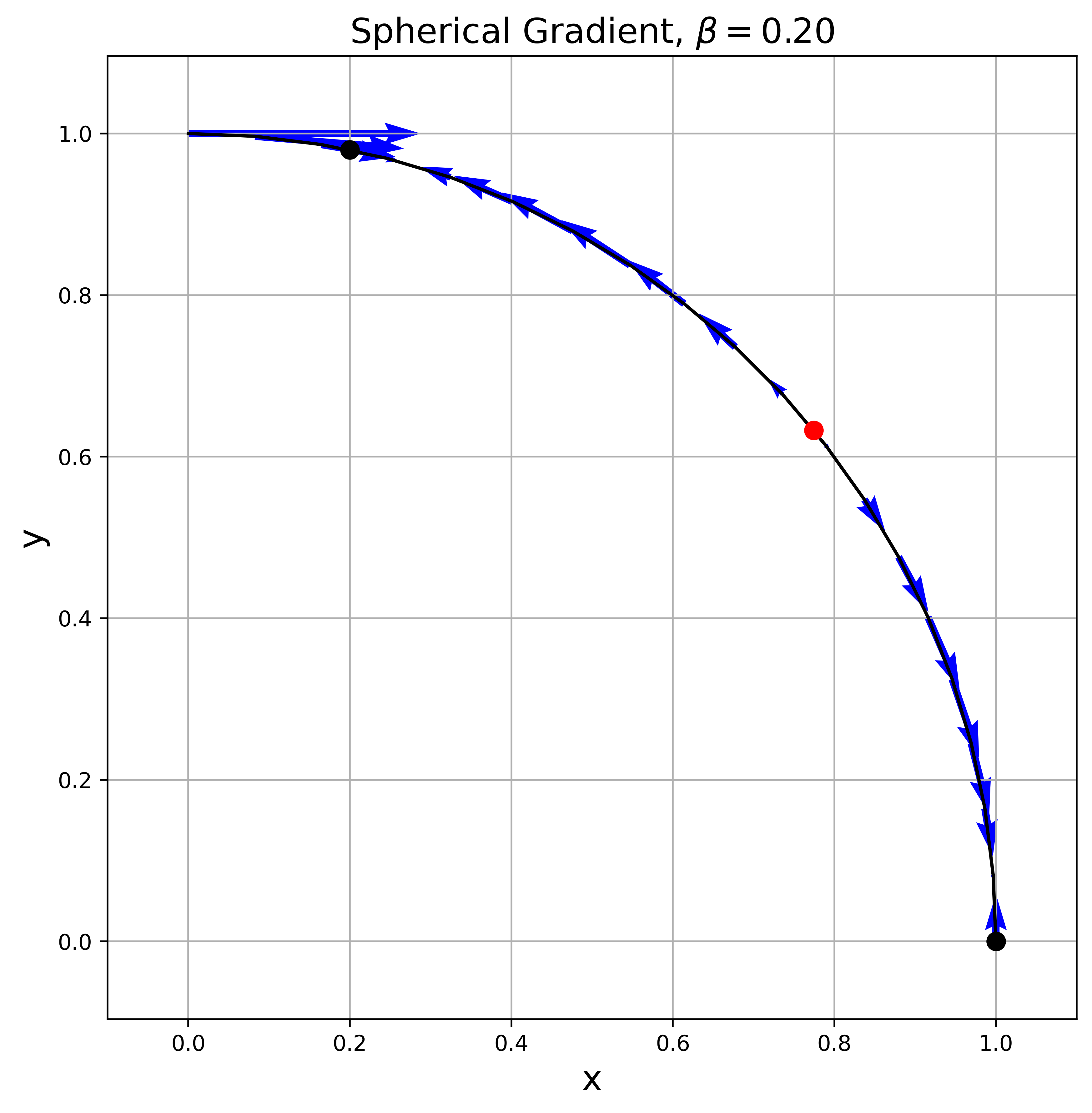}} 
     {\includegraphics[width=0.3\textwidth]{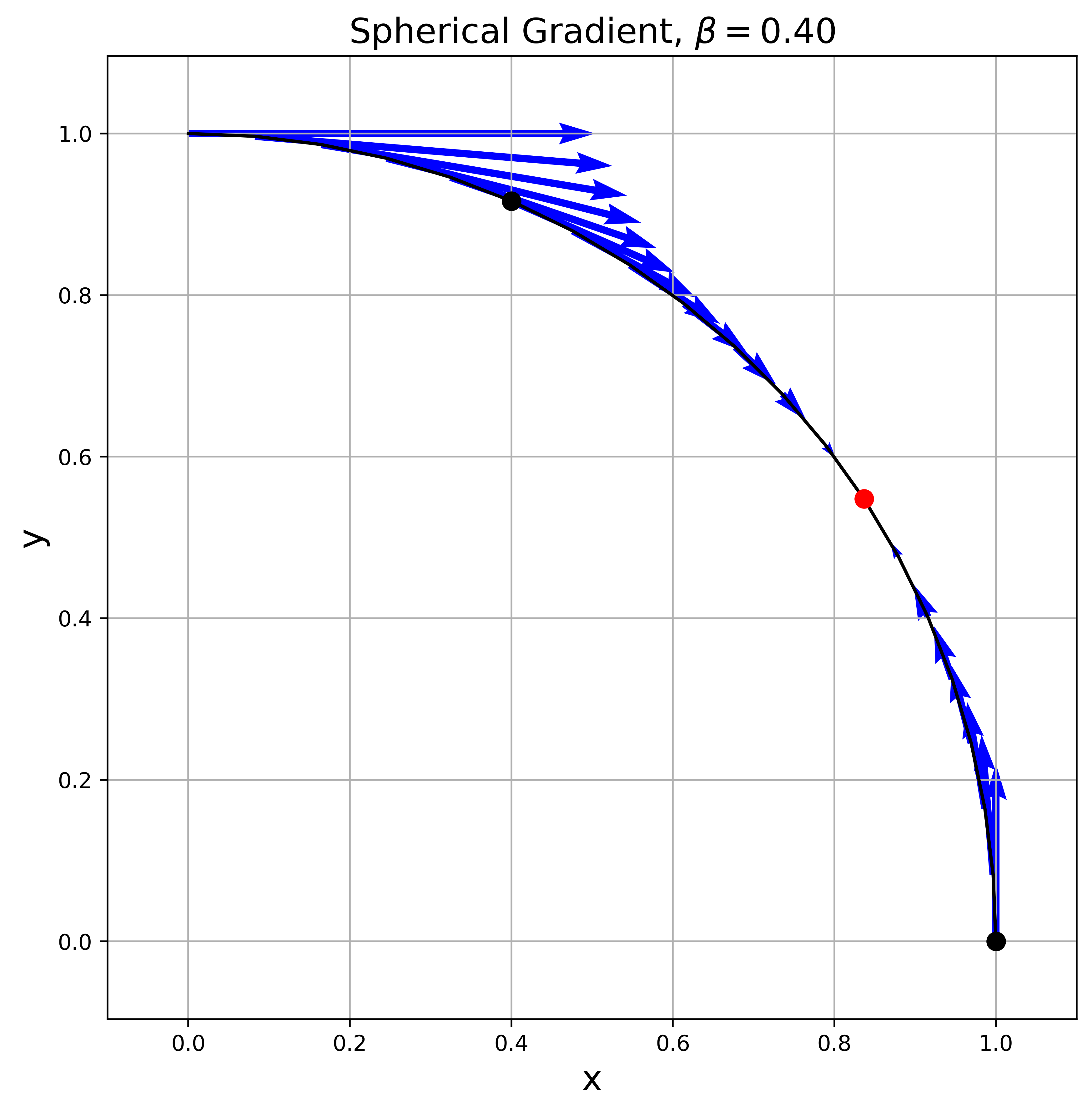}} \\
     \vspace{0.3cm}
    \caption{\textit{The spherical gradient flow vector field, as the dot product between the two vectors increases from left to right; $d=k=2$}. The two index vectors are shown as black dots. $\beta \in \{ 0.1, 0.2, 0.4 \}$ as shown in the title. 
    The average fixed point $\bar{\bw}$ (red dot) turns from a maximum (for $\beta \in \{ 0.1, 0.2 \}$) to a minimum for $\beta = 0.4$. The activation function here is $h_3$ hence the saddle-to-minimum transition happens at $\beta_c = 1/3$. 
    }
\end{figure*}

\end{document}